\documentclass[10pt]{article} 
\usepackage{fullpage}

\usepackage{amssymb}            
\usepackage{mathtools}          
\usepackage{mathrsfs}           
\usepackage{graphicx}           
\usepackage{subcaption}         
\usepackage[space]{grffile}     
\usepackage{url}                

\usepackage{xcolor}
\newcount\comments  
\newcommand{\genComment}[2]{\ifnum\comments=1{\textcolor{#1}{\textsf{\footnotesize #2}}}\fi}

\allowdisplaybreaks

\title{A Natural Extension To Online Algorithms For Hybrid RL With Limited Coverage
}

\usepackage{natbib}
\usepackage[T1]{fontenc}
\usepackage[utf8]{inputenc}
\usepackage{authblk}

\usepackage{amsmath,bbm,bm}
\usepackage{amsfonts, amssymb}
\usepackage{amsthm}
\usepackage{mathtools}

\usepackage{algorithm}
\usepackage{algpascal}
\usepackage{algorithmicx}
\usepackage{algpseudocode}
\usepackage{tabularx}
\algdef{SE}[SUBALG]{Indent}{EndIndent}{}{\algorithmicend\ }%
\algtext*{Indent}
\algtext*{EndIndent}

\newcommand{\DE}{\operatorname{DE}}

\newtheorem{thm}{Theorem}
\newtheorem{lem}{Lemma}
\newtheorem{prop}{Proposition}

\newtheorem{aspt}{Assumption}

\newtheorem{defn}{Definition}

\newcommand{\Reg}{\operatorname{Reg}}

\def\eqref#1{equation~\ref{#1}}

\def\1{\bm{1}}

\def\vtheta{{\bm{\theta}}}

\def\vx{{\bm{x}}}

\def\vz{{\bm{z}}}

\def\mI{{\bm{I}}}

\def\mV{{\bm{V}}}

\DeclareMathAlphabet{\mathsfit}{\encodingdefault}{\sfdefault}{m}{sl}
\SetMathAlphabet{\mathsfit}{bold}{\encodingdefault}{\sfdefault}{bx}{n}

\def\gA{{\mathcal{A}}}

\def\gD{{\mathcal{D}}}
\def\gE{{\mathcal{E}}}
\def\gF{{\mathcal{F}}}
\def\gG{{\mathcal{G}}}
\def\gH{{\mathcal{H}}}

\def\gL{{\mathcal{L}}}

\def\gN{{\mathcal{N}}}
\def\gO{{\mathcal{O}}}
\def\gP{{\mathcal{P}}}

\def\gS{{\mathcal{S}}}
\def\gT{{\mathcal{T}}}
\def\gU{{\mathcal{U}}}

\def\gX{{\mathcal{X}}}

\def\sD{{\mathbb{D}}}

\def\sN{{\mathbb{N}}}

\def\sR{{\mathbb{R}}}

\newcommand{\E}{\mathbb{E}}

\newcommand{\R}{\mathbb{R}}

\newcommand{\dimde}{\text{dim}_{\text{DE}}}

\DeclareMathOperator*{\argmin}{arg\,min}

\newcommand{\off}{\operatorname{off}}
\newcommand{\on}{\operatorname{on}}

\begin{document}

\author[1]{Kevin Tan
\thanks{Equal contribution.}}
\author[2]{Ziping Xu $^*$}
\affil[1]{Department of Statistics and Data Science,
The Wharton School, University of Pennsylvania}
\affil[2]{Department of Statistics,
Harvard University}

\maketitle

\begin{abstract}
Hybrid Reinforcement Learning (RL), leveraging both online and offline data, has garnered recent interest, yet research on its provable benefits remains sparse. Additionally, many existing hybrid RL algorithms \citep{song2023hybrid, nakamoto2023calql, amortila2024harnessing} impose coverage assumptions on the offline dataset, but we show that this is unnecessary. A well-designed online algorithm should ``fill in the gaps'' in the offline dataset, exploring states and actions that the behavior policy did not explore.
Unlike previous approaches that focus on estimating the offline data distribution to guide online exploration \citep{li2023reward}, we show that a natural extension to standard optimistic online algorithms -- warm-starting them by including the offline dataset in the experience replay buffer -- achieves similar provable gains from hybrid data even when the offline dataset does not have single-policy concentrability. We accomplish this by partitioning the state-action space into two, bounding the regret on each partition through an offline and an online complexity measure, and showing that the regret of this hybrid RL algorithm can be characterized by the best partition -- despite the algorithm not knowing the partition itself. As an example, we propose DISC-GOLF, a modification of an existing optimistic online algorithm with general function approximation called GOLF used in \cite{jin2021bellmaneluder, xie2022role}, and show that it demonstrates provable gains over both online-only and offline-only reinforcement learning, with competitive bounds when specialized to the tabular, linear and block MDP cases. Numerical simulations further validate our theory that hybrid data facilitates more efficient exploration, supporting the potential of hybrid RL in various scenarios.
\end{abstract}


\section{Introduction}
\label{sec:intro}


Reinforcement Learning (RL) encompasses two main approaches: online and offline. Online RL involves agents learning to maximize rewards through real-time interactions with their environment, essentially learning by doing. Conversely, offline RL involves agents learning optimal actions by analyzing data collected by others, akin to learning by observation. 
However, learning by both watching and doing, or learning from both offline pre-collected data and online exploration, often called hybrid RL, remains underexplored. Despite recent scholarly attention, \citep{song2023hybrid,nakamoto2023calql,wagenmaker2023leveraging,xie2022policy,li2023reward,amortila2024harnessing}, only \cite{wagenmaker2023leveraging} and \cite{li2023reward} consider the case where the \textbf{offline dataset may not have single-policy concentrability}.\footnote{An offline complexity measure that measures the coverability of the offline dataset \citep{zhan2022offline} with respect to the state-and-action pairs covered by a single reference policy.}

\cite{li2023reward} suggest dividing the state and action space $\gX$ within a tabular MDP into a disjoint partition $\gX_{\off} \oplus \gX_{\on} = \gX$. The intuition is as follows. If the offline dataset has sufficient coverage of the state and action pairs in $\gX_{\off}$, a good algorithm should direct its online exploration to sufficiently explore $\gX_{\on}$. Previous approaches \citep{li2023reward,wagenmaker2023leveraging} solve difficult optimization problems with the Frank-Wolfe algorithm to perform reward-free online exploration of the under-covered portion of the state and action space. These approaches are not generally applicable to existing state-of-the-art online algorithms for deep RL, and so we take a different approach.

Many online algorithms explore by maintaining an experience replay buffer, minimizing the empirical risk over it to sequentially update estimates about the unknown environment \citep{auer2008near}. One may trivially include the offline dataset in the experience buffer to obtain a hybrid RL algorithm, as others have previously noted \citep{song2023hybrid, nakamoto2023calql, amortila2024harnessing}, under coverage assumptions on the offline dataset.\footnote{Unlike these, we are able to include the entire offline dataset -- we do not need to discard any offline samples.}

Though being extensively applied in empirical studies, {\it it is not clear whether (1) simply appending the offline dataset to the experience replay buffer can lead to a provable improvement when the offline dataset is of poor quality, or (2) whether it ensures sufficient exploration for the portion of the state-action space without good coverage.} We seek to address this gap in our paper, tackling the more difficult setting where the offline data may be of arbitrarily poor quality without single-policy concentrability, in the context of regret-minimizing online RL with general function approximation. To our knowledge, we are the first to do so. 

\vspace{-3mm}
\paragraph{Our Contributions.}
We address this gap by modifying an optimistic algorithm for general function approximation algorithm called GOLF (introduced in \cite{jin2021bellmaneluder} and used in \cite{xie2022policy}). 
We show that a hybrid version of GOLF (which we call DISC-GOLF) that simply includes an offline dataset in the parameter estimation achieves a provable improvement in the regret bound over pure online and offline learning, even when the offline dataset has poor coverage.

This is done through considering {\it arbitrary} (not necessarily disjoint) partitions of the state-action space $\gX_{\off} \cup \gX_{\on} = \gX$. We bound the regret by the coverage of the behavior policy on the offline partition $\gX_{\off}$  and a complexity measure for online learning on the online partition $\gX_{\on}$. We then show that the overall regret of a hybrid algorithm can be characterized by the regret bound on the best possible partition -- despite the algorithm not knowing the partition itself.\footnote{This is similar in spirit to the adaptivity that \cite{li2023reward} showed for the tabular PAC RL case, but with a far more complicated algorithm that requires data splitting, behavior cloning, and reward-free exploration.}  This analysis yields a general recipe for initializing generic online RL algorithms with offline data of arbitrarily poor quality, that we hope may be of use to other researchers seeking to derive similar algorithms. 

We specialize this bound to the tabular, linear, and block MDP cases, achieving competitive sample complexities in each. Numerical simulations demonstrate that hybrid RL indeed encourages exploration of the region of the state-action space that is not well-covered by the offline dataset.

\section{Problem Setup}
We consider the situation where we are given access to a function class $\gF$, and aim to model the optimal Q-function using it. Below, we introduce some notation that we use throughout the paper. 

\vspace{-3mm}
\paragraph{Notation.} Let $\gN_{\gF}(\rho)$ be the $\rho$-covering number of function class $\gF$ w.r.t the supremum norm. Let $N_{\off}$ and $N_{\on}$ (where $N = N_{\off} + N_{\on}$) be the number of episodes in the offline dataset and the number of online episodes respectively. We will use the notation $T=N_{\on}$ interchangeably. For any set $\gX \subset \gS \times \gA \times [H]$, let $\gX_h = \{(s, a) \in \gS \times \gA: (s, a, h) \in \gX\}$, and $\Delta(\gX)$ all distributions over $\gX$.

\vspace{-3mm}

\paragraph{Episodic MDPs.} We consider episodic MDPs denoted by $\{\gS, \gA, H, P, R\}$, where $\gS$ is the state space, $\gA$ the action space, $H$ the horizon, $P = \{P_h\}_{h \in [H]}$ the collection of transition probabilities with each $P_h: \gS \times \gA \mapsto \Delta(\gS)$, and $R = \{R_h\}_{h \in [H]}$ the collection of reward functions with each $R_h: \gS \times \gA \mapsto [0, 1]$. An agent interacts with the environment for $H$ steps within each episode. On the each step $h \in [H]$, the agent observes the current state $s_h \in \gS$ and chooses an action $a_h \in \gA$, and the environment generates the next state $s_{h+1} \sim P_h(\cdot \mid s_h, a_h)$ and the current reward $r_h = R_h(s_h, a_h)$. A policy $\pi$ is a mapping from $\gS$ to $\Delta(\gA)$, the set of distributions over the action space. The function class $\gF$ induces a policy class $\Pi := \{\pi^f : f\in \gF\}$ through the greedy policy with regard to each function $\pi^f$. Throughout the paper, we denote $\gX = \gS \times \gA \times [H]$.

\begin{defn}[Occupancy Measure]
\label{defn:occ}
    The occupancy measure $d^{\pi} = \{d^{\pi}_h\}_{h = 1}^H$ is the collection of state-action distributions induced by running policy $\pi$. We write $\sD$ for the set of all possible $d^\pi$.
\end{defn}

\vspace{-3mm}
\paragraph{Hybrid RL.} We study the natural setting of online fine-tuning given access to an offline dataset, where an agent interacts with the environment for $N_{\on}$ steps given access to an offline dataset $\gD_{\off}$ consisting of $N_{\off}$ episodes. We assume that the offline dataset is collected through some fixed policy $\pi_{\off} = \{\pi_{\off, h}\}_{h \in [H]}$. Let $\mu$ be the occupancy measure induced by $\pi_{\off}$, and denote by $s_{h}^{(t)}$, $a_{h}^{(t)}$ and $r_{h}^{(t)}$ the state, action and reward on step $h \in [H]$ within episode $t \in [N_{\on}]$. The goal of an online RL algorithm is to maximize the cumulative reward $\sum_{t = 1}^{N_{\on}} \sum_{h = 1}^H r_{h}^{(t)}$.

We follow the standard definition of value functions for episodic MDPs. The value function of a policy $\pi$ is $V_{h}^{\pi}(s) = \E_{\pi}[\sum_{h' = h}^H r_{h'} \mid s_{h'} = s]$, where $\E_{\pi}$ denotes the expectation over trajectories induced by taking policy $\pi$. Let $Q^{\pi}_h(s, a) = \E_{\pi}[\sum_{h' = h}^H r_{h'} + V_{h'+1}^{\pi}(s_{h'+1}) \mid s_{h'} = s, a_{h'} = a]$, where we set $V_{H+1}^{\pi}(s) \equiv 0$. Write $V^{\star}$ and $Q^{\star}$ for the optimal value and Q-functions. The cumulative regret of an online algorithm $\gL$ is
$
    \Reg(N_{\on}, \gL) = \E_{\gL}\left[\sum_{t = 1}^{N_{\on}}\left( V^\star_{1}(s_{1}^{(t)}) - \sum_{h = 1}^H r_{h}^{(t)} \right)\right],
$
where $\gL : \gH \to \Pi$ is any learning algorithm that maps all the previous observations, i.e. the history $\gH$, to a policy, and $\E_{\gL}$ denotes the expectation over all the trajectories generated by the interaction between algorithm $\gL$ and the underlying MDP.

\vspace{-3mm}
\paragraph{Function Approximation.} We approximate the optimal Q-function with a function class $\gF = \{\gF_h\}_{h \in [H]}$, where each $\gF_h \subseteq [0, H]^{\gS \times \gA}$. The Bellman operator for each $h \in [H-1]$ is
$
    \gT_h f_{h+1}(s, a) \coloneqq R_h(s, a) + \E_{s' \sim P_h(\cdot \mid s, a)}\left[ \max_{a' \in \gA} f_{h+1}(s', a')\right].
$
We further define the Bellman error w.r.t $f \in \gF$ by $\gE_h f  = \gT_h f_{h+1} - f_{h}$ and the squared Bellman error by $\gE_h^2 f  = (\gT_h f_{h+1} - f_{h})^2$. For a distribution $d \in \Delta(\gS \times \gA)$, we write $\|f_h - \gT_h f_{h+1}\|^2_{2, d} = \E_{(s_h, a_h) \sim d} [\gE^2_h f]$.
Below, we make the following routine assumption on the richness of the function class \citep{liu2020provably, rajaraman2020fundamental, rashidinejad2023bridging, uehara2023pessimistic}. This may be relaxed to the weaker related notion of realizability as in \cite{zanette2023realizability} at the cost of an amplifying factor dependent on the metric entropy of the function class, dataset coverage, and the discrepancy between $\gF$ and its image under the Bellman operator, but this is outside the scope of our analysis. 

\begin{aspt}[Bellman Completeness]
    We assume that for all $f_{h+1} \in \gF_{h+1}$, $\gT_h f_{h+1} \in \gF_h$. Note that this implies realizability: $Q^*_h \in \gF_h$.
\end{aspt}


\section{Measures of Complexity}

In this section, we extend existing complexity measures for offline and online learning with general function approximation in order to use them to understand the complexity of hybrid RL. We will use each on an arbitrary partition of the state-action space, with the intuition being that the offline complexity measure should characterize the difficulty of learning only on the portion that is well-covered by the behavior policy, and the online complexity measure for the difficulty of learning on the portion that has not been explored yet. We later show that a subsequent regret bound can be determined by the complexity measures over any partition, and so the regret is characterized by the infimum over the partitions of the complexity measures on them.


\vspace{-3mm}
\paragraph{Offline Complexity Measures.} In offline RL, the sample complexity is bounded by the notion of concentrability \citep{xie2021bellman}. For a function class on Bellman error $\gG$ and a reference policy $\pi$, the all-policy and single-policy concentrability \citep{zhan2022offline} are defined as:
\begin{equation*}
    c_{\off}(\gF, \pi) := \max_h \sup_{f \in \gF} \frac{\|f_h - \gT_h f_{h+1}\|_{2, d_h^\pi}^2}{\|f_h - \gT_h f_{h+1}\|_{2, \mu_{h}}^2}, \text{ and } c_{\off}(\gF) \coloneqq \sup_{\pi} c_{\off}(\gF, \pi). \label{equ:off_c}
\end{equation*}
There is an algorithm \citep{xie2021bellman} that finds an $\epsilon$-optimal policy in $\tilde{\gO}(c_{\off}(\gF, \pi^\star) / \epsilon^2)$ episodes.

\vspace{-3mm}
\paragraph{Online Complexity Measures.} To characterize the online complexity measure, we extend a recently proposed measure, the SEC (Sequential Extrapolation Coefficient) from \cite{xie2022role}\footnote{In their paper, the SEC has a $1$ in the denominator instead of $H^2$ because they assume $Q_h \in [0,1]$.}:
\begin{equation*}
    c_{\on}(\gF,T) \coloneqq \max_{h \in [H]} \sup_{\left\{f^{(1)}, \ldots, f^{(T)}\right\} \subseteq \gF} \sup_{(\pi^{(1)}, \dots, \pi^{(T)})} \left\{\sum_{t=1}^T \frac{\E_{d_h^{\pi^{(t)}}}[f_h^{(t)} - \gT_h f_{h+1}^{(t)}]^2}{H^2 \vee \sum_{i=1}^{t-1} \E_{d_h^{\pi^{(i)}}}[(f_h^{(t)} - \gT_h f_{h+1}^{(t)})^2]}\right\}.
    \label{equ:on_c}
\end{equation*}
\cite{xie2022role} provide an online algorithm with a regret bound of the form $\tilde{\gO}(H\sqrt{c_{\on}(\gF, T) \cdot T})$. Similar extensions can be proposed for other online complexity measures.

\vspace{-3mm}
\paragraph{Reduced Complexity Through State-Action Space Partition.} As previously mentioned, a hybrid algorithm can reduce its online learning complexity by exploring what has not been seen in the offline dataset. This motivates us to consider a partition on the state-action space $\gX = \gS \times \gA \times [H]$. We denote the offline and online partition by $\gX_{\off}$ and $\gX_{\on}$, respectively. We define the offline and online partial complexity measure on each partition by
\begin{align*}
    c_{\off}(\gF, \gX_{\off}) 
    &\coloneqq \max_h \sup_{f \in \gF} \frac{\|(f_h - \gT_h f_{h+1}) \mathbbm{1}_{(\cdot, h) \in \gX_{\off}}\|_{2, d_h^\pi}^2}{\|(f_h - \gT_h f_{h+1})\mathbbm{1}_{(\cdot, h) \in \gX_{\off}}\|_{2, \mu_{h}}^2},\\
    c_{\on}(\gF, \gX_{\on},T) 
    &\coloneqq \max_{h \in [H]} \sup_{\left\{f^{(1)}, \ldots, f^{(T)}\right\} \subseteq \gF} \sup_{(\pi^{(1)}, \dots, \pi^{(T)})} \left\{\sum_{t=1}^T \frac{ \E_{d_h^{\pi^{(t)}}}[(f_h^{(t)} - \gT_h f_{h+1}^{(t)}) \mathbbm{1}_{(\cdot, h) \in \gX_{\on}}]^2}{H^2 \vee \sum_{i=1}^{t-1} \E_{d_h^{\pi^{(i)}}}[(f_h^{(t)} - \gT_h f_{h+1}^{(t)})^2 \mathbbm{1}_{(\cdot, h) \in \gX_{\on}}]}\right\}.
\end{align*}


    
\vspace{-3mm}
Viewing $c_{\on}$ and $c_{\off}$ as complexity measures on the function class $\gF_h - \gT_h \gF_{h+1}$ induced by $\gF$ and Bellman operator $\gT$, our partial complexity measures can be seen as restricting this function class such that any function in this class is non-zero only when the input is in $\gX_{\off}$ or $\gX_{\on}$. This leads to smaller complexity measures for both online and online learning. This is not unique to our choices of complexity measures. Other measures in the literature, such as the Rademacher complexity and covering number, also indicate a reduced complexity for $\gF_h - \gT_h \gF_{h+1}$.

\vspace{-3mm}
\paragraph{Partial All-Policy Concentrability Is Less Stringent Than Single-Policy Concentrability.} While \cite{li2023reward} successfully employ a notion of partial single-policy concentrability in the tabular setting, our regret bound depends on the partial all-policy concentrability. This falls short of the notion of partial single-policy concentrability that \cite{li2023reward} successfully employ in the tabular setting. We attribute this to our desire to work with the simple procedure of appending the offline dataset to the experience replay buffer in the context of general function approximation -- our algorithm is much simpler and their techniques, being specialized to the tabular case, cannot be extended to general function approximation. 

However, as our regret bound utilizes the best partition of the state-action space, our result already obtains an improvement over the common requirement of single-policy concentrability \textit{over the entire state-action space} in hybrid RL with general function approximation \citep{song2023hybrid, nakamoto2023calql, amortila2024harnessing}. While the two are not directly comparable, the best partial all-policy concentrability coefficient, which our algorithm uses adaptively, is always finite (we can always take $\gX_{\off}=\emptyset$) even when the single-policy concentrability coefficient is unbounded.



\vspace{-3mm}
\paragraph{Main Result.} Our main novel theoretical result is in showing that the overall regret of a hybrid algorithm (we first show this for DISC-GOLF, then for a general class of online algorithms) can be characterized by $c_{\off}(\gF, \gX_{\off})$ and $c_{\on}(\gF, \gX_{\on}, N_{\on})$ for any (not necessarily disjoint) partition $\gX_{\on}$ and $\gX_{\off}$ -- despite the algorithm not knowing the partition itself. As this holds for every partition, the guarantee we provide therefore incorporates the best possible split without the algorithm having to know or estimate it.

\section{Online Finetuning From Offline Data}

Here is an example. In this section, we derive an efficient regret bound for an optimistic online algorithm with general function approximation that is warm-started with offline data of arbitrarily poor quality. This regret bound demonstrates provable gains over both online-only and offline-only reinforcement learning through splitting the state-action space.\footnote{The algorithm is never aware of the partition. The partition is only a convenient, but useful, theoretical construct.}

\vspace{-3mm}
\paragraph{An Optimistic Hybrid RL Algorithm Warm-Started With Offline Data.}

We modify the GOLF algorithm from \citet{xie2022role} to incorporate a dataset $\gD_{\off}$ collected by a behavior policy $\pi_b$ with occupancy measure $\mu$. We name the resulting algorithm DISC-GOLF.\footnote{Data Informed Sequential Confidence-sets -- Global Optimism based on Local Fitting.} The modification is simple and intuitive -- we simply warm-start the online exploration by appending the offline data to the experience replay buffer at the beginning, and explore from there. Remarkably, this simple modification enables us to deal with an offline dataset that only has partial coverage. To our knowledge, this has only previously been accomplished in the tabular setting with a far more complicated algorithm \citep{li2023reward}.

\begin{algorithm}[h]
    \caption{DISC-GOLF}
    \begin{algorithmic}[1]
        \State {\bfseries Input:} Offline dataset $\gD_{\off}$, samples sizes $N_{\on}$, $N_{\off}$, function class $\gF$ and confidence width $\beta > 0$
        \State {\bfseries Initialize:} $\gF^{(0)} \leftarrow \gF$, $\gD_h^{(0)}\leftarrow \emptyset, \forall h \in [H]$
        \For{episode $t = 1, 2, \dots, N_{on}$}
            \State Select policy $\pi^{(t)} \leftarrow \pi_{f^{(t)}}$, where $f^{(t)}:=\operatorname{argmax}_{f \in \mathcal{F}^{(t-1)}} f_1\left(x_1, \pi_{f, 1}\left(x_1\right)\right)$.
            \State Execute $\pi^{(t)}$ for one episode and obtain trajectory $(s_1^{(t)}, a_1^{(t)}, r_1^{(t)}), \dots, (s_H^{(t)}, a_H^{(t)}, r_H^{(t)})$.
            \State Update dataset $\mathcal{D}_h^{(t)} \leftarrow \mathcal{D}_h^{(t-1)} \cup\{(s_h^{(t)}, a_h^{(t)}, r_h^{(t)}, s_{h+1}^{(t)})\}, \forall h \in[H]$.
            \State Compute confidence set:
            \vspace{-3mm}
            $$
                \mathcal{F}^{(t)} \leftarrow\left\{f \in \mathcal{F}: \mathcal{L}_h^{(t)}\left(f_h, f_{h+1}\right)-\min _{f_h^{\prime} \in \mathcal{F}_h} \mathcal{L}_h^{(t)}\left(f_h^{\prime}, f_{h+1}\right) \leq \beta \quad \forall h \in[H]\right\},
            $$
            \vspace{-5mm}
            $$ \text{where } \mathcal{L}_h^{(t)}\left(f, f^{\prime}\right):=\sum_{\left(s, a, r, s^{\prime}\right) \in \mathcal{D}_h^{(t)} \cup \mathcal{D}_{\off, h}}\left(f(s, a)-r-\max _{a^{\prime} \in \mathcal{A}} f^{\prime}\left(s^{\prime}, a^{\prime}\right)\right)^2, \forall f \in \gF_{h}, f^{\prime} \in \mathcal{F}_{h+1}.
            $$
        \vspace{-5mm}
        \EndFor
    \end{algorithmic}
\label{alg:GOLF}
\end{algorithm}

\vspace{-3mm}
\paragraph{Main Result.}

The following result shows that the regret can be decomposed into two terms that depend on the offline and online complexity measures over the best possible partition of $\gX$. 

\begin{thm}[Regret Bound for DISC-GOLF]
\label{thm:regret_bound}
Let $\gX_{\off}, \gX_{\on}$ be an arbitrary partition over $\gX = \gS \times \gA \times [H]$. 
Algorithm \ref{alg:GOLF} satisfies the following regret bound with probability at least $1-\delta$:
$$
    \Reg(N_{\on})
    = \gO\left(\inf_{\gX_{\on}, \gX_{\off}} \left(\sqrt{\beta H^4N_{\on}\left(\frac{N_{\on}}{N_{\off}}\right) c_{\off}(\gF, \gX_{\off})}  + \sqrt{\beta H^4 N_{\on} c_{\on}(\gF, \gX_{\on}, N_{\on})}\right)\right),
$$
where 
$\beta = c_1\log \left(N H \mathcal{N}_{\mathcal{F}}(1/N) / \delta\right)$ for some constant $c_1$ with $N = N_{\on} + N_{\off}$.\footnote{The online-only bound in \cite{xie2022role} is of the form $\sqrt{\beta H^2 N_{\on} c_{\on}(\gF, \gX, N_{\on})}$, as they assume $Q$-functions are bounded by $[0,1]$, accounting for the remaining $H^2$ dependence.}
\end{thm}

We defer the proof to Appendix \ref{app:proof_thm1}. This shows that an optimistic online RL algorithm can be adapted to the hybrid setting in a very natural way -- initializing it with an offline dataset. Although the algorithm is completely unaware of the partition, the regret bound provides the best regret guarantee over all partitions of the state-action space.

The offline term depends on $N_{\on}\left(\frac{N_{\on}}{N_{\off}}\right)$, and so depends on the ratio of the number of online and offline episodes. However, due to the infimum over partitions, the overall regret bound will always be no worse than $\tilde{\gO}(\sqrt{N_{\on})}$, as when $N_{\on} \gg N_{\off}$ we can simply take $\gX_{\on} = \gX$ to find that $c_{\off}(\gF, \emptyset) = 0$. Conversely, in the few-shot learning setting where $N_{\off} \gg N_{\on}$, the regret bound is approximately $\tilde\gO\left(\sqrt{\beta H^4 N_{\on} c_{\on}(\gF, \gX_{\on}, N_{\on})}\right)$, improving on the GOLF regret of $\tilde\gO\left(\sqrt{\beta H^4 N_{\on} c_{\on}(\gF, \gX, N_{\on})}\right)$. 

This bound roughly matches that of \cite{song2023hybrid, nakamoto2023calql, amortila2024harnessing} in terms of the dependence on horizon and log-covering number. However, unlike these, we do not require single-policy concentrability. The infimum over partitions gives us a finite partial all-policy concentrability coefficient $c_{\off}(\gF, \gX_{\off})$, even when the single-policy concentrability coefficient over the entire space $C^*$ is unbounded. Additionally, these previous approaches discard any offline data beyond the size of the online dataset (i.e. offline datapoints $N_{\on}+1,...,N_{\off}$), and so obtain a guarantee that does not depend on $N_{\off}$. We do not need to discard any offline samples, enabling us to use the offline data in our regret bound.

\section{Case Studies}

Theorem \ref{thm:regret_bound} established a regret bound for the general function approximation setting. Throughout this section, we examine case studies to demonstrate the exact improvement of hybrid RL algorithm over pure online and pure offline algorithms and characterize the set of good partitions. We defer all proofs in this section to Appendix \ref{app:case_studies}.

\subsection{Tabular MDPs.}
The most commonly considered MDP family is that of the Tabular MDPs, with a finite number of states and actions. As each $Q$ function at the step $h$ can be represented as a $|\gS| \times |\gA|$ dimensional vector, we consider the function class $\gF_h = [0,  H]^{|\gS||\gA|}$. For a constant $\rho > 0$, an intuitive choice of partition that corresponds closely to the choice of \cite{li2023reward} is
$\gX_{\off}(\rho) \coloneqq \{(s, a, h): \sup_{\pi} {d^{\pi}_h(s, a)}/{\mu_h(s, a)} \leq \rho \}.$ As such, the partial offline concentrability coefficient reduces to the supremum of density ratios over the offline partition, allowing us to bound the partial SEC by the cardinality of the online partition. 

\begin{prop}
\label{prop:case_tabular}
    We can bound $c_{\off}(\gF, \gX_{\off}) \leq \sup_{\pi }\sup_{(s, a, h) \in \gX_{\off}} \frac{d_h^{\pi}(s, a)}{\mu_h^{\pi}(s, a)}=\sup_{\pi }\left\lVert\frac{d_h^{\pi}\mathbbm{1}_{\gX_{\off}}}{\mu_h^\pi}\right\rVert_{\infty}$ and $c_{\on}(\gF, \gX_{\on}) \lesssim \max_{h \in [H]}|\gX_{\on, h}| \log(N_{\on})$. As such, with probability at least $1-\delta$, 
    $$\Reg(N_{\on})
    = \tilde{\gO}\left(\inf_{\gX_{\on}, \gX_{\off}} \left(\sqrt{H^5SAN_{\on}\left(\frac{N_{\on}}{N_{\off}}\right)\sup_{\pi }\left\lVert\frac{d_h^{\pi}\mathbbm{1}_{\gX_{\off}}}{\mu_h^{\pi}}\right\rVert_{\infty}}  + \sqrt{H^5SA\max_{h \in [H]}|\gX_{\on}|N_{\on}}\right)\right).
    $$
\end{prop}

Therefore, if the offline dataset has good coverage on a subset $\gX_{\off}$, the complexity of online learning complexity can be reduced to the cardinality of its complement $\gX_{\on}$. We then obtain a regret bound that is at most a factor of $H^2SA$ off from the minimax-optimal results in the offline-only and online-only cases \citep{rashidinejad2023bridging, shi2022pessimistic, azar2017minimax, xie2022policy}, even though (1) DISC-GOLF is a very general model-free function-approximation algorithm, and (2) we did not perform a specialized analysis of this case beyond simply bounding the partial SEC in this setting. We anticipate that analyzing specialized versions of DISC-GOLF can achieve tighter sample complexities in the same sense that \cite{li2023qlearning} accomplish for Q-learning. Note that in a few shot learning setting, where $N_{\off} \gg N_{\on}$, the regret is approximately $\tilde\gO\left(\sqrt{H^5SA \max_h|\gX_{\on, h}| N_{\on} \log(N_{\on})}\right)$, where $\gX_{\on}$ is the set of state, action and step tuples where the offline occupancy measure $\mu$ is unsupported.

\subsection{Linear MDPs.}
The family of Linear MDPs is a common MDP family that generalizes the tabular case, defined in Definition \ref{defn:linear_MDP}. It can be shown that the linear function class for action-value function approximation: $\gF_h = \{\langle \phi(\cdot), w_h\rangle: w_h \in \sR^d, \|w_h\| \leq 2H\sqrt{d}\}$ is Bellman complete \citep{jin2020provably}.

\begin{defn}[Linear MDP]
\label{defn:linear_MDP}
An episodic MDP is a linear MDP with a feature map $\phi: \mathcal{S} \times$ $\mathcal{A} \rightarrow \mathbb{R}^d$, if for any $h \in[H]$, there exist $d$ unknown (signed) measures $\boldsymbol{\nu}_h=(\nu_h^{(1)}, \ldots, \nu_h^{(d)})$ over $\mathcal{S}$ and an unknown vector $\boldsymbol{\theta}_h \in \mathbb{R}^d$, such that for any $(s, a) \in \mathcal{S} \times \mathcal{A}$, we have
$
{P}_h(\cdot \mid s, a)=\left\langle{\phi}(s, a), \boldsymbol{\nu}_h(\cdot)\right\rangle \text{ and } r_h(s, a)=\left\langle{\phi}(s, a), {\vtheta}_h\right\rangle,
$
where $\|\phi(s, a)\|_2 \leq 1$ for all $s, a$ and $\max\{\|\boldsymbol{\nu}_h(\gS)\|, \|\vtheta_h(\gS)\| \leq \sqrt{d}\}$ for all $h \in [H]$.
\end{defn}

We can define a partition of the state-action space $\gX$ as follows. For any subset $\gX' \subset \gS \times \gA$, consider the image of the feature map $\phi(\gX') = \{\phi(s, a): (s, a) \in \gX'\}$. We can choose $\Phi_{\off} \subseteq \sR^{d}$ and $\Phi_{\on} \subseteq \sR^{d}$ to be the subspaces spanned by $(\phi(\gX_{\on, h}))_{h\in[H]}$ and $(\phi(\gX_{\off, h}))_{h\in[H]}$, with dimensions $d_{\off}$ and $d_{\on}$ respectively. That is, any partition of the state-action space $\gX$ induces two subspaces of $\R^d$ through the feature map $\phi$. Let $\gP_{\off}$ and $\gP_{\on}$ be the orthogonal projection operators onto $\Phi_{\off}$ and $\Phi_{\on}$. We can then upper bound the complexity measures over each partition, as we show in Proposition \ref{prop:case_linear}.

\begin{prop}
\label{prop:case_linear}
    Let $\phi_{\off} = \gP_{\off} \phi$. We have $c_{\off}(\gF, \gX_{\off}) \leq \max_h 1/\lambda_{d_{\off}}(\E_{\mu_h}[ \phi_{\off}\phi_{\off}^\top])$ and $c_{\on}(\gG_{\on}) = \gO(d_{\on} \log(H N_{\on}) \log(N_{\on}))$, where $\lambda_n$ is the $n$-th largest eigenvalue. Then, with probability at least $1-\delta$, the regret $\Reg(N_{\on})$ is bounded by
    $$\Reg(N_{\on}) = \tilde\gO\left(\inf_{\gX_{\on}, \gX_{\off}} \left(\sqrt{dH^5N_{\on}\left(\frac{N_{\on}}{N_{\off}}\right)\max_h \frac{1}{\lambda_{d_{\off}}(\E_{\mu_h}[ \phi_{\off}\phi_{\off}^\top])}}  + \sqrt{d_{\on}dH^5N_{\on}}\right)\right).$$
\end{prop}

We can compare this result to the $\sqrt{d^2H^3N_{\on}}$ minimax lower bound from \cite{zhou2021nearly}, and the best known upper bound from \cite{zanette2020learning} of $\sqrt{d^2H^4N_{\on}}$, for online RL in linear MDPs. It is exciting to note that by incorporating offline data into an online algorithm, we can improve the dependence on dimension of the regret incurred on the online partition from $d^2$ to $d_{\on}d$. We accomplish this by bounding the SEC in the linear MDP case by $d_{\on}$, up to logarithmic factors. This therefore demonstrates another example of provable gains from hybrid RL.

\subsection{Block MDPs.}
A block MDP (BMDP) refers to an environment with a finite but unobservable latent state space $\gU$, a finite action space $\gA$, and a possibly infinite but observable state space $\gS$ \citep{dann2019oracleefficient, misra2019kinematic, du2021provably}. At each step, the environment generates a current state $s_h \sim q(\cdot \mid u_h)$ given the underlying latent state $u_h \in \gU$. This is described by the block structure outlined below.

\begin{defn}[Block Structure]
    \label{defn:block_MDP}
    A block MDP is an MDP where each context $x \in \gX$ uniquely determines its generating state $u \in \gU$, i.e. there is a decoding function $f^*: \gS \mapsto \gU$ such that $q(\cdot \mid u)$ is supported on $(f^*)^{-1}(u)$.
\end{defn}

Any partition $\gX_{\off}, \gX_{\on}$ induces a partition on the latent state-action space $\bar{\gX}_{\off} = \{(f^*(s), a, h): (s, a, h) \in \gX_{\off}\}$ and $\bar{\gX}_{\on} = \{(f^*(s), a, h): (s, a, h) \in \gX_{\on}\}$, and the offline behavior policy and a given policy $\pi$ induce measures $\bar{\mu}_h$ and $\bar{d}_h^{\pi}$ on $\gU \times \gA$. Then,  Proposition \ref{prop:BMDP_bound} shows that the offline and online learning complexities are determined by the cardinalities of the induced partitions of the latent state space. This bound is also dependent on $\beta$, but we omit it in the main text for brevity.

\begin{prop}
\label{prop:BMDP_bound}
    In a block MDP, $c_{\off}(\gF, \gX_{\off}) \leq \sup_{\pi }\sup_{(u, a, h) \in \bar{\gX}_{\off}} \frac{\bar{d}_h^{\pi}(u, a)}{\bar{\mu}_h^{\pi}(u, a)}$ and $c_{\on}(\gF, \gX_{\on}, T) = \gO(\max_{h}|\bar{\gX}_{\on, h}| \log(N_{\on}))$ if $\gF$ is Bellman-complete. 
    Then, with probability at least $1-\delta$,
    $$
    \Reg(N_{\on})
    = \tilde{\gO}\left(\inf_{\gX_{\on}, \gX_{\off}} \left(\sqrt{ H^4N_{\on}\left(\frac{N_{\on}}{N_{\off}}\right) \sup_{\pi }\sup_{(u, a, h) \in \bar{\gX}_{\off}} \frac{\bar{d}_h^{\pi}(u, a)}{\bar{\mu}_h^{\pi}(u, a)}}  + \sqrt{H^4N_{\on}\max_{h}|\bar{\gX}_{\on, h}| }\right)\right).
    $$
\end{prop}

\section{A Recipe for General Algorithms}

The analysis and techniques used above are by no means applicable only to DISC-GOLF. 
In Proposition \ref{prop:general_recipe} below, we provide a general recipe that can be used to analyze how a general online algorithm $\gL$ can benefit from being initialized with access to an offline dataset. 

We define $d_h^{(t)}$ to be the measure over $\gS \times \gA$ induced by running algorithm $\gL$ for $t$ iterations at horizon $h$. This bound depends on a set of error terms $\delta_h^t$, which for example is (1) the Bellman error $f_h^t - \gT_h f_{h+1}^t$ in the case of general function approximation with DISC-GOLF, (2) the sum of upper confidence bonus terms,  estimation errors, and two martingale terms with UCBVI \citep{azar2017minimax} for the tabular setting, and (3) the gap multiplied by the probability each arm is pulled in the bandit case with UCB \citep{auer2003ucb}. We then have the following result below that provides a guarantee for the procedure of ``hybridifying'' general online algorithms by initializing them with offline datasets. We defer the proof of Proposition \ref{prop:general_recipe} to Appendix \ref{app:general_case}.


\begin{prop}
    \label{prop:general_recipe}
    Let $\gL$ be a general online learning algorithm that satisfies the following conditions:
    \begin{enumerate} 
        \itemsep0em 
        \item $\gL$ admits the regret decomposition $\Reg_{\gL}(T) \leq \sum_{t = 1}^T \sum_{h = 1}^H \E_{(s, a) \sim d^{(t)}_h}[\delta_h^{t}(s, a)]$ for some collection of random functions\footnote{This is often the Bellman error in the case of MDPs.} $(\delta_h^{t})_{h = 1}^H$ with each $\delta_h^t$ a mapping from $\gX \mapsto \sR$; 
        \item 
        $
            \sum_{t = 1}^T \sum_{h = 1}^H \left(N_{\off} \E_{(s, a) \sim \mu_h}[\delta_h^t(s, a)^2] + \sum_{i = 1}^{t-1} \E_{(s, a) \sim d^{i}_h}[ \delta_h^{t}(s, a)^2]\right) \leq \beta(\delta, H)
        $ w.p. $1-\delta$;
        \item there exists a function $c_{\on}: \gP(\gX) \times \sN$ such that for any $\gX' \subset \gX$, it holds with a probability at least $1-\delta$ that
        $
            \sum_{t = 1}^T \sum_{h = 1}^H \E_{(s, a) \sim d^{(t)}_h}[\delta_h^t(s, a) \mathbbm{1}{(x, a, h) \in \gX'}] = \gO(c_{\on}(\gX', T) H^{\gamma} \beta(\delta, H) T)^{\xi},
        $ for some $\xi \in (0, 1)$, $\gamma \in \mathbb{Z}_{\geq 0}$, and where $\beta:(0, 1) \mapsto \sR$ is some measure of complexity of the algorithm and its dependence on the probability of failure $\delta$;
        \item a coverage measure on any $\gX' \subset \gX$ of $
            c_{\off}(\gX') \coloneqq \sup_{h \in [H]} \sup_{\pi} \frac{\E_{d_h^{\pi}}[\delta_h^t(s, a) \mathbbm{1}(s,a,h \in \gX')]}{\E_{\mu_h}[\delta_h^t(s, a)\mathbbm{1}(s,a,h \in \gX')]} \footnote{We set $0/0$ as 0.}.
        $
    \end{enumerate}
    Then, the algorithm $\gL$ satisfies the following regret bound w.p. at least $1-\delta$:
    $$
        \Reg_{\gL}(T) = \gO\left(\inf_{\gX_{\on}, \gX_{\off}}(c_{\on}(\gX_{\on}, T) \beta(\delta, H)H^\gamma T)^{\xi}  + H \sqrt{\beta(\delta, H) \cdot c_{\off}(\gX_{\off})\cdot \frac{N_{\on}^2}{N_{\off}}}\right).
    $$
\end{prop}

Informally, Proposition \ref{prop:general_recipe} states that given (1) a regret decomposition over the errors at each timestep, (2) a bound on the in-sample error (or just the error under the behavior policy measure), (3) an online-only regret bound for the original algorithm, and (4) an offline coverage measure, we can provide a similar guarantee to what we showed for DISC-GOLF in Theorem \ref{thm:regret_bound}. We anticipate that one can use this or similar arguments to improve upon the minimax-optimal online-only and offline-only regret bounds when analyzing more specialized algorithms.

\section{Numerical Experiments}

To illustrate the notion that appending the offline dataset to the experience replay buffer can encourage sufficient exploration for the portion of the state-action space that does not have good coverage, we perform two simulation studies in the tabular and linear MDP settings respectively.

\subsection{Forest, Tabular MDP.}
\begin{figure}[t!]
    \centering
    \includegraphics[scale=0.4]{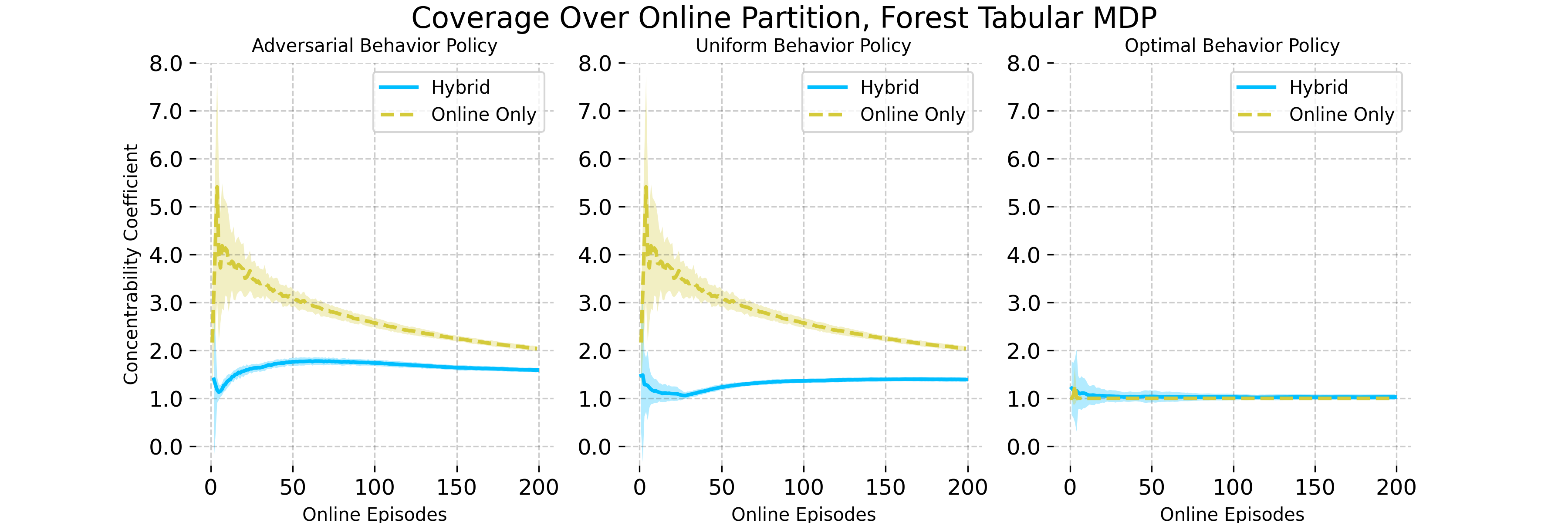}
    \includegraphics[scale=0.4]{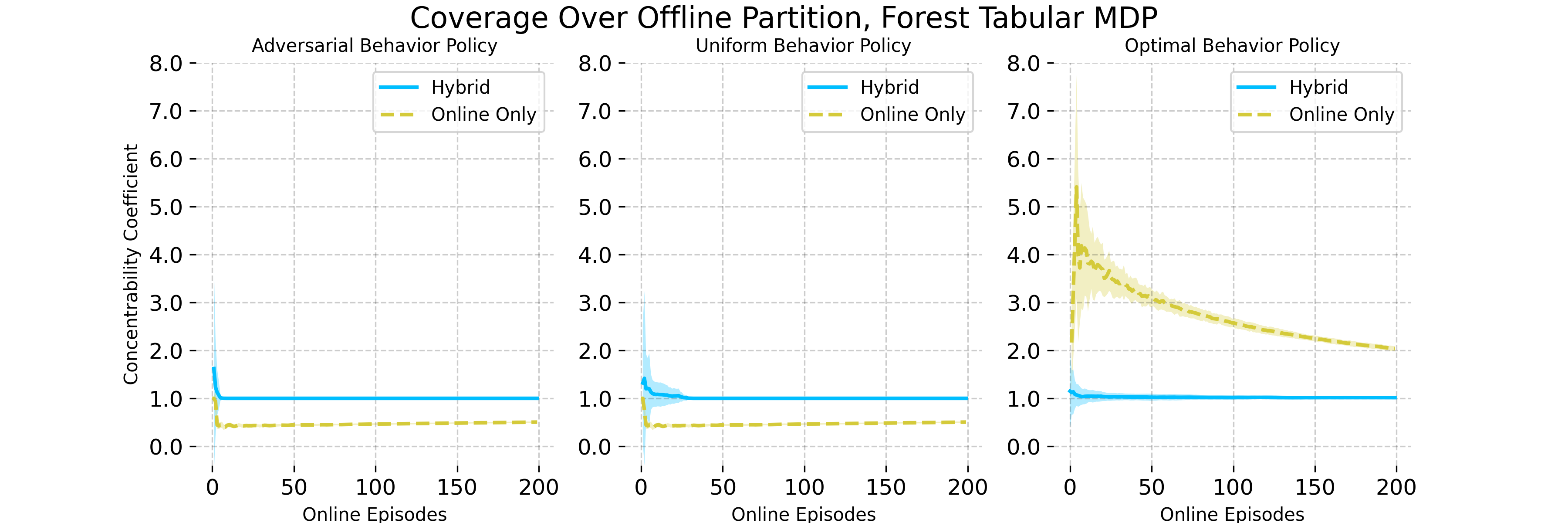}
    \caption{Coverage of the online samples averaged over 30 trials, with $1.96\hat\sigma$ confidence intervals. Hybrid RL explores more of the online partition and less of the offline partition than online RL when the behavior policy is poor, and vice-versa when the behavior policy is good. Lower is better.}
    \label{fig:tabular-coverage}
\end{figure}

We used a simple forest management simulator from the \texttt{pymdptoolbox} package of \cite{pymdptoolbox}. This environment has $4$ states and $2$ actions, and we used a horizon of $20$ years. Every year, the agent can choose to wait and let the forest grow, earning a reward of $4$ if the forest is $3$ years old and $0$ otherwise, or cut the forest down, earning a reward of $1$ if the forest is between $1-2$ years old, $2$ if the forest is $3$ years old, and $0$ otherwise. The forest burns down with $0.1$ probability each year (making it $0$ years old). 

We examine how an optimistic model-based algorithm, UCBVI \citep{azar2017minimax}, behaves when warm-started with an offline dataset. We considered three behavior policies -- adversarial, uniform, and optimal. The adversarial behavior policy does the opposite of the optimal policy $60\%$ of the time, and takes a random action $40\%$ of the time. Each offline dataset consisted of $100$ trajectories. The offline partition was chosen to be the state-action pairs with occupancy at least $1/SA$, and the online partition was defined as its complement. In Figure \ref{fig:tabular-coverage}, we plot the full and partial single-policy concentrability coefficients from running UCBVI on each partition and for each behavior policy. Between this and Figure \ref{fig:tabular-visits} in Appendix \ref{app:experiments}, which depicts the cumulative visits to each partition, we see that when the behavior policy is poor or middling, hybrid RL explores more of the online partition to fill in the gaps in the offline dataset than online RL does. However, when the behavior policy is optimal, hybrid RL sticks to the online partition due to the warm-started model estimation.

\subsection{Tetris, Linear MDP.}

\begin{figure}[t!]
    \centering
    \includegraphics[scale=0.4]{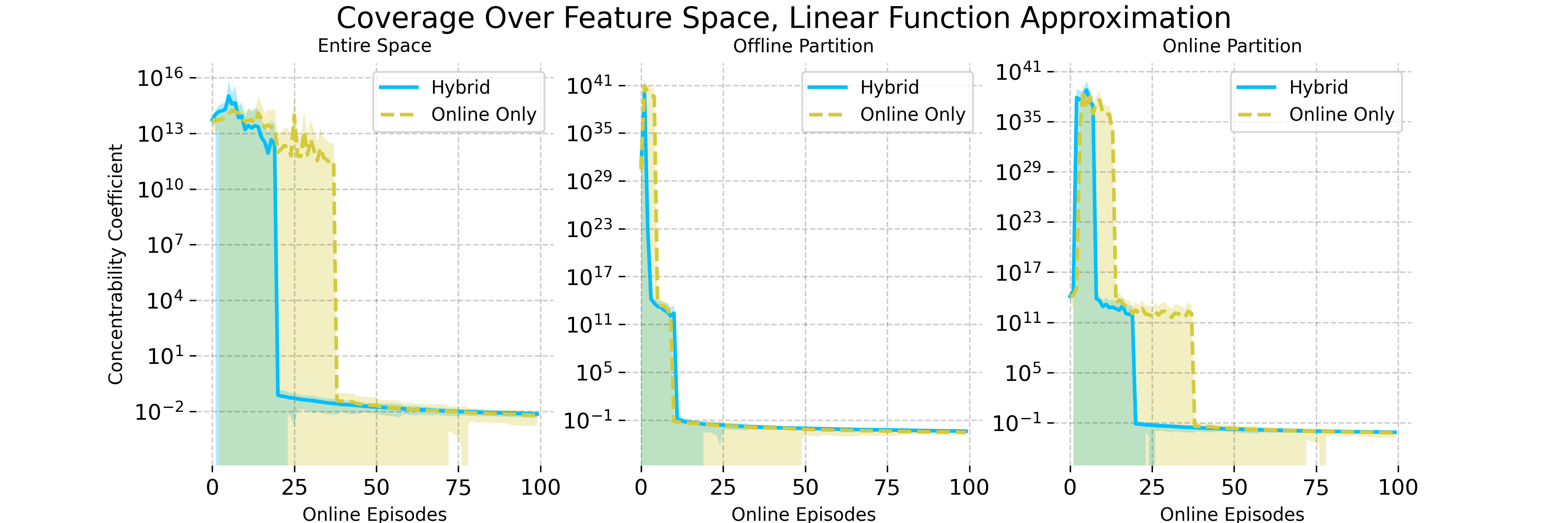}
    \caption{Plot of the full and partial all-policy concentrability coefficients of the online samples from $100$ online episodes. The solid line represents the mean over $30$ trials, and the shaded areas represent confidence intervals generated by $1.96$ times the sample standard deviation. We see that hybrid RL takes fewer online episodes than online-only RL to achieve a lower concentrability coefficient.}
    \label{fig:linear-coverage}
\end{figure}

In another experiment, we consider a scaled-down version of Tetris with pieces of shape at most $2\times2$, where the game board has a width of $6$. The agent can take four actions, corresponding to the degree of rotation in $90$ degree intervals, at each timestep. The reward is the negative of any additional increase in the height of the stack beyond $2$. We examine the extent to which an optimistic RL algorithm, LSVI-UCB from \cite{jin2020provably}, explores the feature space more effectively when initialized with an offline dataset of 200 trajectories of length 40 from a uniform behavior policy. 

Due to combinatorial blowup, this environment is rather difficult to explore. We therefore chose to focus on the portion of the environment that was covered by the uniform behavior policy within the $8000$ simulated timesteps in the offline dataset. This was accomplished through projecting the $640$-dimensional one-hot state-action encoding into a 60-dimensional subspace estimated through performing SVD on the offline dataset. The offline partition was chosen to be the span of the top $5$ eigenvectors, while the online partition was the span of the remaining 55. Without the projection, the results are qualitatively similar to what we have observed, except with concentrability coefficients that are orders of magnitudes higher. 

In Figure \ref{fig:linear-coverage}, we plot the all-policy concentrability coefficients from $n=1,...,N_{\on}$, given by the largest, $k$-th largest, and $d-k$-th largest eigenvalues of the data covariance matrix and its projections onto the offline and online partitions respectively. We see that the concentrability coefficients on the entire space, as well as the offline and online partitions, decrease much faster with the hybrid algorithm than that of the online-only algorithm. This further confirms that an online algorithm initialized with a precollected offline dataset can explore more effectively.

\section{Conclusion and Discussion}

We have answered through theoretical results and numerical simulations that simply appending the offline dataset to the experience replay buffer can (1) lead to an improvement when the offline dataset is of poor quality, and (2) encourage sufficient exploration for the portion of the state-action space without good coverage. This yields a general recipe for modifying existing online algorithms to incorporate offline data, and we propose DISC-GOLF, a modification of an existing optimistic online algorithm, as an example, with promising theoretical guarantees demonstrating provable gains over both offline-only and online-only learning.

\vspace{-3mm}
\paragraph{Limitations and Future Work.} 
Due to our desire to work with the simple procedure of appending the offline dataset to the experience replay buffer with general function approximation, our regret bound depends on partial all-policy concentrability. This is not bad, as the best partial all-policy concentrability coefficient is always finite (as we can always take $\gX_{\off}=\emptyset$) even when the single-policy concentrability coefficient is unbounded. Still, improving this to a guarantee based on partial single-policy concentrability would be valuable.

As GOLF, and therefore DISC-GOLF, uses the squared Bellman error, we (1) require completeness \citep{xie2022role}, and (2) incur a total $H^4$ dependence before any additional penalties from the log-covering number of the function class.\footnote{\cite{xie2022role} work with Q-functions bounded in $[0,1]$ instead of $[0,H]$, so their bound depends on $H^2$.} We and \cite{xie2022role} use this instead of the average Bellman error to facilitate change-of-measure arguments. If one could work with the average Bellman error without a change-of-measure, one could potentially reduce the dependence to $H^3$ while only requiring realizability, but it is not clear whether this can be accomplished. 

Practical and computationally tractable adaptations of DISC-GOLF can be developed in the same sense as \citep{cheng2022adversarially, nakamoto2023calql}, including approaches to optimism in deep RL such as the optimistic actor-critic of \cite{ciosek2019better}. One could extend the theoretical analyses in this paper to practical algorithms in deep RL. 

Hybrid RL poses a unique opportunity to bypass the pitfalls of offline reinforcement learning. We address the issue of coverage in this work, but strategically collected online data may also help to solve other pertinent issues in offline RL such as distribution shift \citep{song2023hybrid, cheng2022adversarially, kumar2020conservative}, or confounding and partial observability \citep{wang2020provably, kausik2023offline, brunssmith2023robust, lu2023pessimism}. 

Finally, while DISC-GOLF uses optimistic online exploration, previous work and our general recipe in Proposition \ref{prop:general_recipe} shows it is possible to be pessimistic \citep{nakamoto2023calql}, or neither \citep{song2023hybrid}. Further analysis on the relative merits of each, or even switching between them as \cite{moskovitz2022tactical} do, is welcomed.


\bibliography{main}
\bibliographystyle{apalike}

\newpage

\appendix

\section{Proof of Theorem 1}
\label{app:proof_thm1}

{\bf Theorem \ref{thm:regret_bound}.}{\it \quad 
    Let $\gX_{\off}, \gX_{\on}$ be an arbitrary partition over $\gX = \gS \times \gA \times [H]$. Note that this partition induces the restricted function classes on the Bellman error $\gG_{\off}$ and $\gG_{\on}$.
    Algorithm \ref{alg:GOLF} satisfies the following regret bound with probability at least $1-\delta$:
    $$
        \Reg(N_{\on})
        = \gO\left(\inf_{\gX_{\on}, \gX_{\off}} \left(\sqrt{\beta H^4\frac{N_{\on}^2}{N_{\off}} c_{\off}(\gF, \gX_{\off})}  + \sqrt{\beta H^4 N_{\on} c_{\on}(\gF, \gX_{\on}, N_{\on})}\right)\right),
    $$
    where 
    $\beta = c_1\left(\log \left[N H \mathcal{N}_{\mathcal{F}}(1/N) / \delta\right]\right)$ for some constants $c_1$ with $N = N_{\on} + N_{\off}$.
}



Let $\gX_{\on}, \gX_{\off}$ be an arbitrary (not necessarily disjoint) partition of $\gS \times \gA \times [H]$. We will bound the regret for an arbitrary partition, allowing us to take the infimum over partitions for the final regret bound.

We first address some notation. Recall that we defined the Bellman error by $\gE f_h = f_h - \gT_h f_{h+1}$.
Following the proof in \citet{xie2022role}, we use the same shorthand for the Bellman error $\delta_h^{(t)} := f_h^{(t)} - \gT_h f_{h+1}^{(t)}$, and the cumulative in-sample occupancy measures (without and with the offline dataset) by
$$
    \bar{d}_h^{(t)} \coloneqq \sum_{i = 1}^{t-1} d^{(i)}_h, \;\; \text{ and } \quad \tilde{d}_h^{(t)} \coloneqq  N_{\off} \mu_{h} +  \sum_{i = 1}^{t-1} d_h^{(i)},
$$
where $d^{(t)}_h = d^{f^{(t)}}_h$ is the occupancy measure induced by running the greedy policy w.r.t $f^{(t)}$. We further write 
$$
    \E_{\bar d_h^{(t)}}[f] = \sum_{i = 1}^t \E_{d_h^{(i)}}[f], \text{ and } \E_{\tilde{ d}_h^{(t)}}[f] = \sum_{i = 1}^t \E_{d_h^{(i)}}[f] + N_{\off} \E_{\mu_{h}}[f].
$$

We require the following lemma to bound the in-sample Bellman error. This is very similar to Lemma 15 of \cite{xie2022role}, except that this incorporates the offline data as well. Note that \cite{xie2022role} work with Q-functions bounded in $[0,1]$ instead of $[0,H]$, so their bound depends on $\beta$ instead of $H^2\beta$. The proof can be found in Appendix \ref{lem:good_event_proof}

\begin{lem}
\label{lem:good_event}
With a probability at least $1-\delta$, for all $t \in [N_{\on}]$, we have that for all $h=1,...,H$
\begin{equation*}
    \text{(i) } Q^*_h \in \mathcal{F}^{(t)}_h, \text{ (ii) } \E_{\tilde{d}_h^{(t)}}[(\delta_h^{(t)})^2] \leq O(H^2\beta),
\end{equation*}
by choosing $\beta = c_1\left(\log \left[N H \mathcal{N}_{\mathcal{F}}(\rho) / \delta\right]+ N \rho\right)$ for some constant $c_1$. 
\end{lem}

With this, we can begin the proof. By a regret decomposition (Lemma 3 \citep{xie2022role}), the total regret can be upper bounded by
\begin{align}
    \Reg(T) \leq \sum_{t = 1}^T \sum_h \E_{(s, a) \sim d_h^{(t)}} [\delta_h^{(t)}(s, a)]. \label{equ:reg_decp}
\end{align}

We further decompose this decomposition (\ref{equ:reg_decp}) by the partition on $\gX$:
$$
    \Reg(T) \leq \sum_{t = 1}^T \sum_h \E_{(s, a) \sim d_h^{(t)}} [\delta_h^{(t)}(s, a) \mathbbm{1}_{(s, a, h) \in \gX_{\on}}] + \sum_{t = 1}^T \sum_h \E_{(s, a) \sim d_h^{(t)}} [\delta_h^{(t)}(s, a) \mathbbm{1}_{(s, a, h) \in \gX_{\off}}],
$$
where we call the first term the online term and the second term the offline term. 

We will bound each term individually. The bound on the online term follows from an argument from \cite{xie2022role}, while the bound on the offline term can be obtained in a reasonably similar way, from applying Cauchy-Schwarz, performing a change of measure, and finally bounding the result by the partial concentrability coefficient. 

Going forward, we will adopt the shorthand $\delta_{h, \on}^{(t)}(s, a) \coloneqq \delta_{h}^{(t)}(s, a) \mathbbm{1}_{(s, a, h) \in \gX_{\on}}$ and $\delta_{h, \off}^{(t)}(s, a) \coloneqq \delta_{h}^{(t)}(s, a) \mathbbm{1}_{(s, a, h) \in \gX_{\off}}$. 

As mentioned above, we upper bound the the first term on the RHS in the same way \citet{xie2022role} do for the online exploration, with the SEC:
\begin{equation*}
    \begin{aligned}
    &\sum_{t = 1}^T \sum_h \E_{d_h^{(t)}} [\delta_{h, \on}^{(t)}]\\
    =& \sum_{t = 1}^T \sum_h \E_{d_h^{(t)}} \left[\delta_{h, \on}^{(t)} \cdot  \left(\frac{H^2 \vee \sum_{i=1}^{t-1}\E_{d_h^{(i)}}[(\delta_{h, \on}^{(t)})^2]}{H^2 \vee \sum_{i=1}^{t-1}\E_{d_h^{(i)}}[(\delta_{h, \on}^{(t)})^2]}\right)^{1 / 2} \right]\\
    \leq& \sqrt{\sum_{t = 1}^T \sum_h \frac{\E_{d_h^{(i)}}[\delta_{h, \on}^{(t)}(s, a)]^2}{H^2 \vee \sum_{i=1}^{t-1}\E_{d_h^{(i)}}[(\delta_{h, \on}^{(t)})^2]} } 
    \sqrt{\sum_{t = 1}^T \sum_{h}\left(H^2 \vee \sum_{i=1}^{t-1}\E_{d_h^{(i)}}[(\delta_{h, \on}^{(t)})^2]\right)} \\
    & \leq \sqrt{H \cdot c_{\on}(\gF, \gX_{\on}, T)} \cdot \sqrt{HT \cdot H^2\beta} \\
    \lesssim& H^2\sqrt{\beta \cdot c_{\on}(\gF, \gX_{\on}, T) \cdot T}.
\end{aligned}
\end{equation*}
The second-last line follows from bounding the term on the left of the third line by the SEC, and bounding the term on the right with Lemma \ref{lem:good_event}. 

We bound the regret incurred by state and actions in $\gX_{\off}$ directly by the offline data. We first perform a similar Cauchy-Schwarz and change of measure argument to before:
\begin{align*}
    &\sum_{t = 1}^T \sum_h \E_{d_h^{(t)}} [\delta_{h, \off}^{(t)}]\\
    =& \sum_{t = 1}^T \sum_h \E_{(s, a) \sim d_h^{(t)}}\left[\delta_{h, \off}^{(t)}(s, a)\left(\frac{ N_{\off} \E_{\mu_{h}}[(\delta_{h, \off}^{(t)})^2] + \sum_{i=1}^{t-1}\E_{d_h^{(i)}}[(\delta_{h, \off}^{(t)})^2]}{ N_{\off} \E_{\mu_{h}}[(\delta_{h, \off}^{(t)})^2] + \sum_{i=1}^{t-1}\E_{d_h^{(i)}}[(\delta_{h, \off}^{(t)})^2]}\right)^{1 / 2}\right]  \\
    \leq& \sqrt{\sum_{t = 1}^T \sum_h \E_{(s, a) \sim d_h^{(t)}} \frac{\left( \delta_{h, \off}^{(t)}(s, a)\right)^2}{ N_{\off} \E_{\mu_{h}}[(\delta_{h, \off}^{(t)})^2] + \sum_{i=1}^{t-1}\E_{d_h^{(i)}}[(\delta_{h, \off}^{(t)})^2]}} \\&\qquad \cdot\sqrt{\sum_{h,t}  \left(  N_{\off} \E_{\mu_{h}}[(\delta_{h, \off}^{(t)})^2] + \sum_{i=1}^{t-1}\E_{d_h^{(i)}}[(\delta_{h, \off}^{(t)})^2]\right)}.
\end{align*}

We can bound the first term with the partial all-policy concentrability coefficient. As for any $h, t$ it holds that
$$\frac{\E_{d_h^{(t)}}[( \delta_{h, \off}^{(t)}(s, a))^2]}{ N_{\off} \E_{\mu_{h}}[(\delta_{h, \off}^{(t)})^2] + \sum_{i=1}^{t-1}\E_{d_h^{(i)}}[(\delta_{h, \off}^{(t)})^2]} \leq \frac{\E_{d_h^{(t)}}[(\delta_{h, \off}^{(t)}(s, a))^2]}{ N_{\off} \E_{\mu_{h}}[(\delta_{h, \off}^{(t)})^2]},$$
this reduces to the partial all-policy concentrability coefficient.


\begin{align*}
    & \sqrt{\sum_{t = 1}^T \sum_h \E_{(s, a) \sim d_h^{(t)}} \frac{\left( \delta_{h, \off}^{(t)}(s, a)\right)^2}{ N_{\off} \E_{\mu_{h}}[(\delta_{h, \off}^{(t)})^2] + \sum_{i=1}^{t-1}\E_{d_h^{(i)}}[(\delta_{h, \off}^{(t)})^2]}}  \\
    \leq& \sqrt{\sum_{h,t}^{H,T} \frac{\E_{d_h^{(t)}}[( \delta_{h, \off}^{(t)})^2]}{ N_{\off} \E_{\mu_{h}}[(\delta_{h, \off}^{(t)})^2]}}  \\
    \leq& \sqrt{\frac{HT}{N_{\off}} \sup_h \sup_{f \in \gF_{h}} \sup_{\pi}\frac{\E_{d_h^{\pi}}[(f_h-\gT_h f_{h+1})^2 \mathbbm{1}_{(\cdot, h) \in \gX_{\off}}]}{\E_{\mu_{h}}[(f_h-\gT_h f_{h+1})^2 \mathbbm{1}_{(\cdot, h) \in \gX_{\off}}]}}  \\
    \leq& \sqrt{\frac{HT}{N_{\off}} c_{\off}(\gF, \gX_{\off})}
\end{align*}

To bound the second term, we use the in-sample regret bound from Lemma \ref{lem:good_event}. We then obtain:
\begin{align*}
    &\sqrt{\sum_{h,t}  \left(  N_{\off} \E_{\mu_{h}}[(\delta_{h, \off}^{(t)})^2] + \sum_{i=1}^{t-1}\E_{d_h^{(i)}}[(\delta_{h, \off}^{(t)})^2]\right)} \\
    \leq& \sqrt{\sum_{h,t}  \left(  N_{\off} \E_{\mu_{h}}[(\delta_{h}^{(t)})^2] + \sum_{i=1}^{t-1}\E_{d_h^{(i)}}[(\delta_{h}^{(t)})^2]\right)}\\
    =& \sqrt{\sum_{h,t}^{H,T} \E_{\tilde{d}_h^{(t)}}[(\delta_h^{(t)})^2]}  \\
    \lesssim& H\sqrt{H T \beta}.
\end{align*}

Putting it all together, 
\begin{align*}
    \Reg(T) 
    &\leq \sum_{t = 1}^T \sum_h \E_{(s, a) \sim d_h^{(t)}} [\delta_h^{(t)}(s, a) \mathbbm{1}_{(s, a, h) \in \gX_{\on}}] + \sum_{t = 1}^T \sum_h \E_{(s, a) \sim d_h^{(t)}} [\delta_h^{(t)}(s, a) \mathbbm{1}_{(s, a, h) \in \gX_{\off}}] \\
    &\lesssim H^2\sqrt{\beta \cdot c_{\on}(\gF, \gX_{\on}, T) \cdot T} + \sqrt{H \cdot c_{\on}(\gF, \gX_{\on}, T)} \cdot \sqrt{HT \cdot H^2\beta}.
\end{align*}

We therefore have the following regret bound:
\begin{align*}
    \Reg(T)
    &= \gO\left(\sqrt{\beta H^4T^2 c_{\off}(\gF, \gX_{\off})/ N_{\off}}  + \sqrt{\beta H^4 T c_{\on}(\gF, \gX_{\on}, T)}\right),
\end{align*}
where we set $\beta = c_1\left(\log \left[N H \mathcal{N}_{\mathcal{F}}(\rho) / \delta\right]+ N \rho\right)$  for some constants $c_1$. Finally, we choose $\rho$ to be $1/N$ so $N\rho$ becomes a constant, to obtain our result.


\newpage
\section{Proofs About The SEC From \cite{xie2022role}}

We first prove a general result, which will be used for various case studies. \cite{xie2022role} have shown that SEC can be bounded by the Distributional-Eluder dimension (Definition \ref{defn:DE_dim}).

\begin{defn}
    \label{defn:DE_dim}
    The Distributional-Eluder dimension $dim_{\operatorname{DE}}(\gG, \sD, \epsilon)$ is the largest $n \in \sN$, such that there exist sequences $\{d^{(1)}, \dots, d^{n}\} \subset {\sD}$ and $\{g^{(1), \dots, g^{(n)}}\}$ such that for all $t \in [n]$, 
    $$        \left|\mathbb{E}_{d^{(t)}}\left[g^{(t)}\right]\right|>\varepsilon^{(t)}, \quad \text{and} \quad \sqrt{\sum_{i=1}^{t-1}\left(\mathbb{E}_{d^{(i)}}\left[g^{(t)}\right]\right)^2} \leq \varepsilon^{(t)},
    $$
    for $\epsilon^{(1)}, \dots, \epsilon^{(n)} \geq \epsilon$. 
\end{defn}

\begin{lem}[Modified from Proposition 13 and 14 \citep{xie2022role}] 
\label{lem:general_SEC}
When we restrict the SEC to the online partition $\gX_{\on}$ to obtain the online complexity measure $c_{\on}(\gF, \gX_{\on}, T)$ on the online partition of $\gX = \gS \times \gA \times [H]$, we have that:
\begin{align*}
    &(1) \quad \quad c_{\on}(\gF, \gX_{\on}, T) \lesssim \log(T) \max_{h} \inf_{\mu_h \in \Delta(\gS \times \gA)} \sup_{\pi} \frac{\E_{(s, a) \sim d_h^{\pi}}[\mathbbm{1}_{(s, a, h) \in \gX_{\on}}]}{\E_{(s, a) \sim \mu_h^{\pi}}[\mathbbm{1}_{(s, a, h) \in \gX_{\on}}]} \\
    &(2) \quad \quad c_{\on}(\gF, \gX_{\on}, T) \lesssim \inf_{\epsilon > 0}\{\epsilon^2T + \max_h \dim_{\operatorname{DE}}((\gF_h - \gT_h \gF_{h+1})\mathbbm{1}_{(\cdot, h) \in \gX_{\on}}, \sD_h, \epsilon)\}\log(T),
\end{align*}
where $(\gF_h - \gT_h \gF_{h+1})\mathbbm{1}_{(\cdot, h) \in \gX_{\on}} \coloneqq \{(s, a) \mapsto (f_h(s, a) - \gT_h f_{h+1}(s, a)) \mathbbm{1}_{(s, a, h) \in \gX_{\on}}: f_h \in \gF_h, f_{h+1} \in \gF_{h+1}\}$ and $\sD_h = \{d^{\pi}_h: \pi \in \Pi\}$. That is, the restricted SEC is bounded by a modified analogue of the coverability coefficient and the Distributional-Eluder dimension, modulo a logarithmic factor.
\end{lem}

\begin{proof}
    The proof is modified from the proof of Proposition 13 in \cite{xie2022role}. Our task is to ensure that the statements in the above two propositions still hold when we restrict the complexity measure to the online partition, and modify the definition of the coverability coefficient. We first prove the statement (1). Similarly to \cite{xie2022role}, we define 
    $$\mu^*_h := \argmin_{\mu_h \in \Delta(\gS \times \gA)} \sup_{\pi} \frac{\E_{(s, a) \sim d_h^{\pi}}[\mathbbm{1}_{(s, a, h) \in \gX_{\on}}]}{\E_{(s, a) \sim \mu_h^{\pi}}[\mathbbm{1}_{(s, a, h) \in \gX_{\on}}]}.$$ 

    We will denote
    $$c_h(\gX_{\on}) := \inf_{\mu_h \in \Delta(\gS \times \gA)} \sup_{\pi} \frac{\E_{(s, a) \sim d_h^{\pi}}[\mathbbm{1}_{(s, a, h) \in \gX_{\on}}]}{\E_{(s, a) \sim \mu_h^{\pi}}[\mathbbm{1}_{(s, a, h) \in \gX_{\on}}]}.$$
    
    We want to show that for any $T>0$, 
    $$c_{\on}(\gF, \gX_{\on}, T) \lesssim \log(T) \max_{h} \inf_{\mu_h \in \Delta(\gS \times \gA)} \sup_{\pi} \frac{\E_{(s, a) \sim d_h^{\pi}}[\mathbbm{1}_{(s, a, h) \in \gX_{\on}}]}{\E_{(s, a) \sim \mu_h^{\pi}}[\mathbbm{1}_{(s, a, h) \in \gX_{\on}}]}.$$

    Recall that 
    $$c_{\on}(\gF, \gX_{\on},T) 
    \coloneqq \max_{h \in [H]} \sup_{\left\{f^{(1)}, \ldots, f^{(T)}\right\} \subseteq \gF} \sup_{(\pi^{(1)}, \dots, \pi^{(T)})} \left\{\sum_{t=1}^T \frac{ \E_{d_h^{\pi^{(t)}}}[(f_h^{(t)} - \gT_h f_{h+1}^{(t)}) \mathbbm{1}_{(\cdot, h) \in \gX_{\on}}]^2}{H^2 \vee \sum_{i=1}^{t-1} \E_{d_h^{\pi^{(i)}}}[(f_h^{(t)} - \gT_h f_{h+1}^{(t)})^2 \mathbbm{1}_{(\cdot, h) \in \gX_{\on}}]}\right\}.$$ 
    
    Unlike \cite{xie2022role}, we will prove this for an arbitrary $h=1,...,H$, and take the maximum over $h$ over both sides of the inequality to obtain our desired result. We therefore fix $h \in [H]$ and consider arbitrary sequences $f^{(1)},...,f^{(T)} \in \gF$ and $\pi^{(t)},...,\pi^{(T)}.$ This therefore induces a sequence of Bellman errors $\delta_h^{(t)}$ for all $h=1,...,H, t=1,...,T$. As in \cite{xie2022role}, we define $\tilde{d}_h^{(t)} := \sum_{i=1}^{t-1} d_h^{\pi^{(t)}}$. 
    
    Consider the stopping time $$\tau(s,a) := \min\left\{t \;\;:\;\; \sum_{i=1}^{t-1} d_h^{\pi^{(t)}}(s,a) \geq \mu^*_h(s,a)\cdot c_h(\gX_{\on}), \; \forall h=1,...,H \right\},$$
    and decompose 
    $$\E_{(s,a) \sim d_h^{\pi^{(t)}}}[\delta_{h, \on}^{(t)}(s,a)] = \E_{(s,a) \sim d_h^{\pi^{(t)}}}[\delta_{h, \on}^{(t)}(s,a)\mathbbm{1}(t < \tau(s,a))] + \E_{(s,a) \sim d_h^{\pi^{(t)}}}[\delta_{h, \on}^{(t)}(s,a)\mathbbm{1}(t \geq \tau(s,a))].$$

    We perform the same Cauchy-Schwarz and change-of-measure argument as in the proof of Theorem \ref{thm:regret_bound} to obtain, writing $d_h^{(t)} = d_h^{\pi^{(t)}},$
    \begin{align*}
        &\sum_{t = 1}^T \frac{\E_{d_h^{(t)}} [\delta_{h, \on}^{(t)}(s, a)]^2}{1 \vee \sum_{i=1}^{t-1}\E_{d_h^{(i)}}[(\delta_{h, \on}^{(t)})^2]} \\
        &\lesssim \sum_{t = 1}^T \frac{\E_{d_h^{(t)}} [\delta_{h, \on}^{(t)}(s, a)\mathbbm{1}(t < \tau(s,a))]^2}{H^2 \vee \sum_{i=1}^{t-1}\E_{d_h^{(i)}}[(\delta_{h, \on}^{(t)})^2]} + \sum_{t = 1}^T \frac{\E_{d_h^{(t)}} [\delta_{h, \on}^{(t)}(s, a)\mathbbm{1}(t < \tau(s,a))]^2}{H^2 \vee \sum_{i=1}^{t-1}\E_{d_h^{(i)}}[(\delta_{h, \on}^{(t)})^2]}.
    \end{align*}

    We tackle the first term representing the burn-in period as follows:
    \begin{align*}
        \sum_{t = 1}^T \frac{\E_{d_h^{(t)}} [\delta_{h, \on}^{(t)}(s, a)\mathbbm{1}(t < \tau(s,a)]^2}{H^2 \vee \sum_{i=1}^{t-1}\E_{d_h^{(i)}}[(\delta_{h, \on}^{(t)})^2]} 
        &= \sum_{t = 1}^T \frac{\E_{d_h^{(t)}} [(\delta_{h, \on}^{(t)}(s, a)/H)\mathbbm{1}(t < \tau(s,a))]^2}{1 \vee \sum_{i=1}^{t-1}\E_{d_h^{(i)}}[(\delta_{h, \on}^{(t)}/H)^2]} \\
        &\leq \sum_{t=1}^T\E_{d_h^{(t)}}[(\delta_{h, \on}^{(t)}(s, a)/H)\mathbbm{1}(t < \tau(s,a))]^2 \\
        &\leq \sum_{t=1}^T\E_{d_h^{(t)}}[(\delta_{h, \on}^{(t)}(s, a)/H)\mathbbm{1}(t < \tau(s,a))]^2 \\
        &\leq \sum_{t=1}^T\E_{d_h^{(t)}}[\mathbbm{1}(t < \tau(s,a))\cdot\mathbbm{1}((s,a) \in \gX_{\on})]^2 \\
        &\leq \sum_{t=1}^T\E_{d_h^{(t)}}[\mathbbm{1}(t < \tau(s,a))\cdot\mathbbm{1}((s,a) \in \gX_{\on})] \\
        &\leq \sum_{t=1}^T\int_{s,a} d_h^{(t)}(s,a) \cdot \mathbbm{1}(t < \tau(s,a))\cdot\mathbbm{1}((s,a) \in \gX_{\on})\\
        &\leq \int_{s,a} \sum_{t=1}^T d_h^{(t)}(s,a) \cdot \mathbbm{1}(t < \tau(s,a))\cdot\mathbbm{1}((s,a) \in \gX_{\on}) \\
        &= \int_{s,a} \tilde{d}_h^{(\tau(s,a))} \cdot \mathbbm{1}((s,a) \in \gX_{\on}) \\
        &= \int_{s,a} (\tilde{d}_h^{(\tau(s,a)-1)} + d_h^{(\tau(s,a)-1)})\cdot \mathbbm{1}((s,a) \in \gX_{\on}) \\
        &\leq 1 + \int_{s,a} c_h(\gX_{\on})\mu_h^*(s,a)\cdot \mathbbm{1}((s,a) \in \gX_{\on}) \\
        &\leq 1 + c_h(\gX_{\on}),
    \end{align*}
    where we divide both the numerator and the denominator by $H^2$ in the first line, use that $\delta_{h, \on}^{(t)} \in [0,H]$ to go from the third to the fourth line, invoke Tonelli's theorem to swap the sum and the integral to go from the sixth to the seventh line, and invoke the definition of $\tau(s,a)$ to bound $\tilde{d}_h^{(\tau(s,a)-1)} \leq c_h(\gX_{\on})\cdot \mu_h^*(s,a)$ and finally observe that $\int_{s,a}d_h^{(\tau(s,a)-1)} = 1$ to go from the third-last to the second-last line.

    Now we tackle the second term. As in \cite{xie2022role}, we observe that
    \begin{align*}
        &\E_{d_h^{(t)}}[\delta_{h, \on}^{(t)}\mathbbm{1}(t \geq \tau(s,a))] \\
        &= \int_{s,a} \mathbbm{1}(t \geq \tau(s,a)) d_h^{(t)}(s,a)\left(\frac{\sum_{i=1}^{t-1} d_h^{(i)}(s,a)}{\sum_{i=1}^{t-1} d_h^{(i)}(s,a)}\right)^{1/2}\delta_{h,\on}^{(t)}\\
        &= \int_{s,a} \mathbbm{1}(t \geq \tau(s,a)) d_h^{(t)}(s,a)\mathbbm{1}((s,a) \in \gX_{\on})\left(\frac{\sum_{i=1}^{t-1} d_h^{(i)}(s,a)\mathbbm{1}((s,a) \in \gX_{\on})}{\sum_{i=1}^{t-1} d_h^{(i)}(s,a)\mathbbm{1}((s,a) \in \gX_{\on})}\right)^{1/2}\delta_{h,\on}^{(t)} \\
        &\leq 
        \sqrt{\int_{s,a}\frac{d_h^{(t)}(s,a)^2\mathbbm{1}(t \geq \tau(s,a))\mathbbm{1}((s,a) \in \gX_{\on}))}{\sum_{i=1}^{t-1} d_h^{(i)}(s,a)\mathbbm{1}((s,a) \in \gX_{\on})}} 
        \sqrt{\left(\sum_{i=1}^{t-1}\E_{d_h^{(i)}}[(\delta_{h, \on}^{(t)})^2]\right)},
    \end{align*}
    as by definition, $\delta_{h,\on}^{(t)}(s,a) = \delta_{h}^{(t)}(s,a)\mathbbm{1}((s,a) \in \gX_{\on})$, and rearrange the inequality in the same way as \cite{xie2022role} to find that
    \begin{align*}
        \frac{\E_{d_h^{(t)}}[\delta_{h, \on}^{(t)}(s,a)\mathbbm{1}(t \geq \tau(s,a))]^2}{H^2 \vee \sum_{i=1}^{t-1}\E_{d_h^{(i)}}[(\delta_{h, \on}^{(t)})^2]}
        &\leq \frac{\E_{d_h^{(t)}}[\delta_{h, \on}^{(t)}(s,a)\mathbbm{1}(t \geq \tau(s,a))]^2}{\sum_{i=1}^{t-1}\E_{d_h^{(i)}}[(\delta_{h, \on}^{(t)})^2]} \\
        &\leq \int_{s,a}\frac{d_h^{(t)}(s,a)^2\mathbbm{1}(t \geq \tau(s,a))\mathbbm{1}((s,a) \in \gX_{\on}))}{\sum_{i=1}^{t-1} d_h^{(i)}(s,a)\mathbbm{1}((s,a) \in \gX_{\on})}.
    \end{align*}

    It then follows that recalling the definition of the stopping time
    $$\tau(s,a) := \min\left\{t \;\;:\;\; \sum_{i=1}^{t-1} d_h^{\pi^{(t)}}(s,a) \geq \mu^*_h(s,a)\cdot c_h(\gX_{\on}), \; \forall h=1,...,H \right\},$$
    that we can bound the post-burn-in term:
    \begin{align*}
        &\sum_{t = 1}^T \frac{\E_{d_h^{(t)}} [\delta_{h, \on}^{(t)}(s, a)\mathbbm{1}(t \geq \tau(s,a))]^2}{H^2 \vee \sum_{i=1}^{t-1}\E_{d_h^{(i)}}[(\delta_{h, \on}^{(t)})^2]} \\
        &\leq \sum_{t = 1}^T \int_{s,a}\frac{d_h^{(t)}(s,a)^2\mathbbm{1}(t \geq \tau(s,a))\mathbbm{1}((s,a) \in \gX_{\on}))}{\sum_{i=1}^{t-1} d_h^{(i)}(s,a)\mathbbm{1}((s,a) \in \gX_{\on})} \\
        &\leq \sum_{t = 1}^T 2\int_{s,a}\frac{d_h^{(t)}(s,a)^2\mathbbm{1}(t \geq \tau(s,a))\mathbbm{1}((s,a) \in \gX_{\on})}{(c_h(\gX_{\on})\cdot\mu_h^*(s,a) + \sum_{i=1}^{t-1} d_h^{(i)}(s,a))\mathbbm{1}((s,a) \in \gX_{\on})} \\
        &\leq \sum_{t = 1}^T 2\int_{s,a}\frac{d_h^{(t)}(s,a)^2\mathbbm{1}((s,a) \in \gX_{\on})}{\left(c_h(\gX_{\on})\cdot\mu_h^*(s,a) + \sum_{i=1}^{t-1} d_h^{(i)}(s,a)\right)\mathbbm{1}((s,a) \in \gX_{\on})} \\
        &\leq \sum_{t = 1}^T 2\int_{s,a}\left(\max_{i\leq T}d_h^{(i)}(s,a)\right)\frac{d_h^{(t)}(s,a)\mathbbm{1}((s,a) \in \gX_{\on})}{\left(c_h(\gX_{\on})\cdot\mu_h^*(s,a) + \sum_{i=1}^{t-1} d_h^{(i)}(s,a)\right)\mathbbm{1}((s,a) \in \gX_{\on})}\\
        &\leq \sum_{t = 1}^T 2c_h(\gX_{\on})\int_{s,a}\mu_h^*(s,a)\frac{d_h^{(t)}(s,a)\mathbbm{1}((s,a) \in \gX_{\on})}{\left(c_h(\gX_{\on})\cdot\mu_h^*(s,a) + \sum_{i=1}^{t-1} d_h^{(i)}(s,a)\right)\mathbbm{1}((s,a) \in \gX_{\on})}\\
        &\leq 2c_h(\gX_{\on})\int_{s,a}\mu_h^*(s,a)\sum_{t = 1}^T\frac{d_h^{(t)}(s,a)\mathbbm{1}((s,a) \in \gX_{\on})}{\left(c_h(\gX_{\on})\cdot\mu_h^*(s,a) + \sum_{i=1}^{t-1} d_h^{(i)}(s,a)\right)\mathbbm{1}((s,a) \in \gX_{\on})} \\
        &\lesssim c_h(\gX_{\on})\int_{s,a}\mu_h^*(s,a)\log(T)\mathbbm{1}((s,a) \in \gX_{\on}) \\
        &\leq c_h(\gX_{\on})\log(T),
    \end{align*}
    where we use the definition of $\mu_h^*$ to bound $\max_{i\leq T} d_h^{(i)}$, the bounded convergence theorem to swap the sum and the integral, and a restricted version of the per-state-action elliptic potential lemma from \cite{xie2022role} in Lemma \ref{lem:restricted-elliptic} to bound $\sum_{t = 1}^T\frac{d_h^{(t)}(s,a)\mathbbm{1}((s,a) \in \gX_{\on})}{\left(c_h(\gX_{\on})\cdot\mu_h^*(s,a) + \sum_{i=1}^{t-1} d_h^{(i)}(s,a)\right)\mathbbm{1}((s,a) \in \gX_{\on})} \lesssim \log(T)$.

    Therefore,
    $$\sum_{t = 1}^T \frac{\E_{d_h^{(t)}} [\delta_{h, \on}^{(t)}(s, a)]^2}{1 \vee \sum_{i=1}^{t-1}\E_{d_h^{(i)}}[(\delta_{h, \on}^{(t)})^2]} \lesssim 1+c_h(\gX_{\on})+c_h(\gX_{\on})\log(T),$$
    so taking the max over all $h=1,...,H$ yields
    $$c_{\on}(\gF, \gX_{\on}, T) \lesssim \log(T) \max_{h} \inf_{\mu_h \in \Delta(\gS \times \gA)} \sup_{\pi} \frac{\E_{(s, a) \sim d_h^{\pi}}[\mathbbm{1}_{(s, a, h) \in \gX_{\on}}]}{\E_{(s, a) \sim \mu_h^{\pi}}[\mathbbm{1}_{(s, a, h) \in \gX_{\on}}]}.$$
\end{proof}

\begin{proof}
    Now, we prove (2). This proof is virtually the same as that of Proposition 14 in \cite{xie2022role}, but we provide it here for completeness. We wish to show that
    $$c_{\on}(\gF, \gX_{\on}, T) \lesssim \inf_{\epsilon > 0}\{\epsilon^2T + \max_h \dim_{\operatorname{DE}}((\gF_h - \gT_h \gF_{h+1})\mathbbm{1}_{(\cdot, h) \in \gX_{\on}}, \sD_h, \epsilon)\}\log(T).$$

    We use the same definition as in \cite{xie2022role}, but specialize it to our context:
    
    \textbf{Generalized $\varepsilon$-(in)dependent sequence.} A distribution $d_h^{(t)}$ is (generalized) $\varepsilon$-dependent on a sequence $\left\{d_h^{(1)}, \ldots, d_h^{(t-1)}\right\}$ if for all $\varepsilon^{\prime} \geq \varepsilon$, if $\left|\mathbb{E}_{d_h}[\delta_{h,\on}^{(t)}]\right|>\varepsilon^{\prime}$ for some $\delta_{h,\on}^{(t)}$, we also have $\sum_{i=1}^{t-1}\left(\mathbb{E}_{d_h^{(i)}}[\delta_{h,\on}^{(t)}]\right)^2>\varepsilon^{\prime 2}$. We say that $d_h^{(t)}$ is (generalized) $\varepsilon$-independent if this does not hold, i.e., for some $\varepsilon^{\prime} \geq \varepsilon$, it has $\left|\mathbb{E}_{d_h^{(t)}}[\delta_{h,\on}^{(t)}]\right|>\varepsilon^{\prime}$ but $\sum_{i=1}^{t-1}\left(\mathbb{E}_{d_h^{(i)}}[\delta_{h,\on}^{(t)}]\right)^2 \leq \varepsilon^{\prime 2}$.

    Note that if $\varepsilon^{\prime} \geq \varepsilon$, then $\varepsilon$-dependent sequence $\Rightarrow \varepsilon^{\prime}$-dependent sequence, and $\varepsilon^{\prime}$-independent sequence $\Rightarrow \varepsilon$-independent sequence.

    The distributional Eluder dimension is the largest $t$ such that $d_h^{(t)}$ is generalized $\epsilon'$-independent of $\left\{d_h^{(1)}, \ldots, d_h^{(t-1)}\right\}$ for some $\epsilon'\geq\epsilon$. We will refer to this $t$ as $t=\dimde(\epsilon)$. This, as in \cite{xie2022role}, upper bounds the lengths of sequences $\left\{d_h^{(1)}, \ldots, d_h^{(\dimde(\epsilon))}\right\}$ such that for all $t=1,...,\dimde(\epsilon)$, $$\left|\mathbb{E}_{d_h^{(t)}}[\delta_{h,\on}^{(t)}]\right|>\varepsilon^{(t)} \text{ and } \sqrt{\sum_{i=1}^{t-1}\left(\mathbb{E}_{d_h^{(i)}}[(\delta_{h,\on}^{(t)})^2]\right)} \leq \varepsilon^{(t)}$$.

    Similarly to \cite{xie2022role}, we define $\beta_h^{(t)} := \sum_{i=1}^{t-1}\mathbb{E}_{d_h^{(i)}}[(\delta_{h,\on}^{(t)})^2]$, and examine
    $$\left\{\frac{\mathbb{E}_{d_h^{(1)}}[\delta_{h,\on}^{(1)}]^2}{1 \vee \beta_h^{(1)}}, \frac{\mathbb{E}_{d_h^{(2)}}[\delta_{h,\on}^{(2)}]^2}{1 \vee \beta_h^{(2)}}, \ldots, \frac{\mathbb{E}_{d_h^{(T)}}[\delta_{h,\on}^{(T)}]^2}{1 \vee \beta_h^{(T)}}\right\},$$
    fixing $\alpha>0$ that we choose later, and writing $L_h^{(t)}$ for the number of disjoint $\alpha\sqrt{1\vee\beta_h^{(t)}}$-dependent subsequences of $d_h^{(t)}$ in $\left\{d_h^{(1)}, \ldots, d_h^{(t-1)}\right\}$. 

    We follow the proof of \cite{xie2022role}. Suppose $\frac{\mathbb{E}_{d_h^{(t)}}[\delta_{h,\on}^{(1)}]^2}{1 \vee \beta_h^{(t)}} > \alpha^2 \implies |\mathbb{E}_{d_h^{(t)}}[\delta_{h,\on}^{(1)}]| > \alpha\sqrt{1\vee\beta_h^{(t)}}$. By definition, there exist at least $L_h^{(t)}$ disjoint subsequences of $\left\{d_h^{(1)}, \ldots, d_h^{(t-1)}\right\}$, which we call $\left(\mathfrak{S}_h^{(1)}, \ldots, \mathfrak{S}_h^{\left(L_h^{(t)}\right)}\right)$, where we have that
    $$\sum_{i=1}^{L_h^{(t)}} \sum_{\nu_h \in \mathfrak{S}_h^{(i)}} \E[\delta_{h,\on}^{(t)}]^2 \geq (1\vee\beta_h^{(t)})\alpha^2,$$
    and by the definition of $\beta_h^{(t)}$, that
    $$\sum_{i=1}^{L_h^{(t)}} \sum_{\nu_h \in \mathfrak{S}_h^{(i)}} \E[\delta_{h,\on}^{(t)}]^2 \leq \sum_{i=1}^{t-1} \E_{d_h^{(i)}}[(\delta_{h,\on}^{(t)})^2]\leq\beta_h^{(t)},$$
    which imply that if $|\mathbb{E}_{d_h^{(t)}}[\delta_{h,\on}^{(1)}]| > \alpha\sqrt{1\vee\beta_h^{(t)}}$ for some $t$, then
    $$\beta_h^{(t)} \geq L_h^{(t)}\left(1 \vee \beta_h^{(t)}\right) \alpha^2 \Longrightarrow L_h^{(t)} \leq \frac{1}{\alpha^2}.$$

    Let $i_1,...,i_k$ be the longest subsequence such that 
    $$\frac{\E_{d_h^{(i_j)}}[\delta_{h,\on}^{(i_j)}]^2}{1\vee\beta_h^{(i_j)}} > \alpha^2.$$

    By the same construction as in \cite{xie2022role}, there exists $j^*$ such that there must exist at least
    $$L^* \geq \frac{k}{\dimde(\alpha)+1}-1$$
    $\alpha$-dependent disjoint subsequences in $\{d_h^{(i_1)},..,d_h^{(i_k)}\}$.

    As for all $\alpha'\geq\alpha$, $\alpha$-dependence implies $\alpha'$-dependence, we have that $L^*\leq \max_t L_h^{(t)}$ after observing also that $\{d_h^{(i_1)},...,d_h^{(i_k)}\} \subset \{d_h^{(1)},...,d_h^{(k)}\}$. Now, observe that
    $$L_h^{(t)} \leq \frac{1}{\alpha^2} \text{ and } L^* \geq \frac{k}{\dimde(\alpha)+1}-1 \implies \frac{1}{\alpha^2}\geq \max_tL_h^{(t)}\geq L^*\geq  \frac{k}{\dimde(\alpha)+1}-1.$$ 

    Now if $\alpha\leq 1$,
    $$
    k \leq\left(1+\frac{1}{\alpha^2}\right)\left(\dimde(\alpha)+1\right) \leq \frac{3 \dimde(\alpha)}{\alpha^2}+1.
    $$
    
    So for any $\epsilon \in(0,1]$, by setting $\alpha=\sqrt{\epsilon}$,
    $$
    \sum_{t=1}^T \mathbbm{1}\left(\frac{\E_{d_h^{(t)}}\left[\delta_{h, \on}^{(t)}\right]^2}{1 \vee \beta_h^{(t)}}>\epsilon\right) \leq \frac{3 \dimde(\sqrt{\epsilon})}{\epsilon}+1 .
    $$
    
    Finally, let 
    $e_h^{(1)},...,e_h^{(T)}$
    be the original sequence of the $$\left\{\frac{\mathbb{E}_{d_h^{(1)}}[\delta_{h,\on}^{(1)}]^2}{1 \vee \beta_h^{(1)}}, \frac{\mathbb{E}_{d_h^{(2)}}[\delta_{h,\on}^{(2)}]^2}{1 \vee \beta_h^{(2)}}, \ldots, \frac{\mathbb{E}_{d_h^{(T)}}[\delta_{h,\on}^{(T)}]^2}{1 \vee \beta_h^{(T)}}\right\}$$ ordered in a decreasing manner. By the same argument as \cite{xie2022role}, for any $c\in (0,1]$, 
    $$\sum_{t=1}^T \frac{\mathbb{E}_{d_h^{(t)}}[\delta_{h,\on}^{(t)}]^2}{1 \vee \beta_h^{(t)}} = \sum_{t=1}^T e_h^{(t)} \leq cT + \sum_{t=1}^T e_h^{(t)}\mathbbm{1}(e_h^{(t)} > c),$$
    and for any $t$ such that $e_h^{(t)} > c$, if we also have that $\eta$ is such that $2\eta\geq e_h^{(t)} > \eta \geq c$, it follows that
    $$t \leq \sum_{i=1}^T \mathbbm{1}(e_h^{(i)} > \eta) \leq \frac{3}{\eta}\dimde(\sqrt{\eta})+1 \leq \frac{3}{\eta}\dimde(\sqrt{c})+1.$$

    We therefore have that $\eta \leq \frac{3\dimde(\sqrt{c})}{t-1}$, and that $e_h^{(t)} \leq \min\left\{\frac{6\dimde(\sqrt{c})}{1}\right\}$. We then have
    $$ \sum_{t=1}^T e_h^{(t)} \mathbb{1}\left(e_h^{(t)}>c\right)  \leq \dimde(\sqrt{c})+\sum_{t=\dimde(\sqrt{c})+1}^T \frac{6 \dimde(\sqrt{c})}{t-1}  \leq \dimde(\sqrt{c})+6 \dimde(\sqrt{c}) \log(T),$$
    which implies that
    $$\sum_{t=1}^T e^{(t)} \leq T c+\dimde(\sqrt{c})+6 \dimde(\sqrt{c}) \log (T).$$

    Finally, choose $c=\epsilon^2$, to find that
    $$c_{\on}(\gF, \gX_{\on}, T) \lesssim \inf_{\epsilon > 0}\{\epsilon^2T + \max_h \dim_{\operatorname{DE}}((\gF_h - \gT_h \gF_{h+1})\mathbbm{1}_{(\cdot, h) \in \gX_{\on}}, \sD_h, \epsilon)\}\log(T).$$
    
\end{proof}

\newpage
\section{Proofs of Case Studies}
\label{app:case_studies}

\subsection{Tabular Case}

\label{app:tabular_case}

{\bf Proposition \ref{prop:case_tabular}.}{\it \quad 
    We can bound $c_{\off}(\gF, \gX_{\off}) \leq \sup_{\pi }\sup_{(s, a, h) \in \gX_{\off}} \frac{d_h^{\pi}(s, a)}{\mu_h^{\pi}(s, a)}=\sup_{\pi }\left\lVert\frac{d_h^{\pi}\mathbbm{1}_{\gX_{\off}}}{\mu_h^\pi}\right\rVert$ and $c_{\on}(\gF, \gX_{\on}) \lesssim \max_{h \in [H]}|\gX_{\on, h}| \log(N_{\on})$. As such, with probability at least $1-\delta$, 
    $$\Reg(N_{\on})
    \lesssim \inf_{\gX_{\on}, \gX_{\off}} \left(\sqrt{H^5SA\log\left(\frac{N}{\delta}\right)\frac{N_{\on}^2}{N_{\off}}\sup_{\pi }\left\lVert\frac{d_h^{\pi}\mathbbm{1}_{\gX_{\off}}}{\mu_h^{\pi}}\right\rVert_{\infty}}  + \sqrt{H^5SA\max_{h \in [H]}|\gX_{\on}|N_{\on}\log\left(\frac{N}{\delta}\right)\log(N)}\right).$$
}

\begin{proof}
    By definition,
    \begin{align*}
        c_{\off}(\gF, \gX_{\off}) 
        =& \sup_{f \in \gF} \sup_{\pi} \frac{\|(f_h - \gT_h f_{h+1})^2 \mathbbm{1}_{(\cdot, h) \in \gX_{\off}}\|^2_{d_h^{\pi}}}{\|(f_h - \gT_h f_{h+1})^2 \mathbbm{1}_{(\cdot, h) \in \gX_{\off}}\|^2_{\mu_h^{\pi}}}\\
        \leq& \sup_{s, a, h \in \gX_{\off}} \frac{d_h^{\pi}(s, a)}{\mu_h^{\pi}(s, a)}.
    \end{align*}
The online complexity measure bound is a direct application of Lemma \ref{lem:general_SEC}:
\begin{align*}
    c_{\on}(\gF, \gX_{\on}, T) 
    & \lesssim 
    \log(T) \max_{h} \inf_{\mu_h \in \Delta(\gS \times \gA)} \sup_{\pi} \frac{\E_{(s, a) \sim d_h^{\pi}}[\mathbbm{1}_{(s, a, h) \in \gX_{\on}}]}{\E_{(s, a) \sim \mu_h^{\pi}}[\mathbbm{1}_{(s, a, h) \in \gX_{\on}}]}\\
    & \leq \log(T) \max_{h} |\gX_{\on, h}|
\end{align*}

Finally, choose $\rho$ to be $1/N$ so $\log(\mathcal{N}_{\gF}(\rho))$ scales with $\log((1/\rho+1)^{HSA})$, which is $HSA\log(N+1)$.

\end{proof}



\newpage
\subsection{Linear Case}
\label{app:linear_case}

{\bf Proposition \ref{prop:case_linear}.}{\it \quad 
    We have $c_{\off}(\gF, \gX_{\off}) \leq \max_h 1/\lambda_{d_{\off}}(\E_{\mu_h}[ \gP_{\off}\phi(s, a) (\gP_{\off}\phi(s, a))^\top]) = \max_h 1/\lambda_{d_{\off}}(\E_{\mu_h}[ \phi_{\off}\phi_{\off}^\top])$ and $c_{\on}(\gG_{\on}) = \gO(d_{\on} \log(H N_{\on}) \log(N_{\on}))$, where $\lambda_n$ is the $n$-th largest eigenvalue of a matrix. Then, with probability at least $1-\delta$, the regret $\Reg(N_{\on})$ is bounded by
    $$\tilde{O}\left(\inf_{\gX_{\on}, \gX_{\off}} \left(\sqrt{dH^5\log(6dNH/\delta)\frac{N_{\on}^2}{N_{\off}}\max_h \frac{1}{\lambda_{d_{\off}}(\E_{\mu_h}[ \phi_{\off}\phi_{\off}^\top])}}  + \sqrt{d_{\on}dH^5N_{\on}\log^3(6dNH/\delta)}\right)\right).$$
}

\begin{proof}
    We first bound the all-policy concentrability for the offline partition. By Bellman completeness, for any $f_{h+1}$, $\gT_h f_{h+1} \in \gF_h$. Since $f_{h+1}$ is parametrized by $w_{h+1}$, we denote by $\gT_h w_{h+1}$ the parameter for $\gT_h f_{h+1}$.
    
    Note that $\|\gT_h w_{h+1}\|_2 \leq 2H\sqrt{d}$ by Assumption \ref{aspt:linear_MDP}.
    \begin{align*}
        c_{\off}(\gF, \gX_{\off}) 
        &= \max_h \sup_{w_h \in \gF_h, w_{h+1} \in \gF_{h+1}} \sup_{\pi} \frac{\E_{d_h^{\pi}}\left[\langle w_h - \gT_h w_{h+1}, \phi(s, a)\rangle \mathbbm{1}_{s, a, h \in \gX_{\off}}\right]^2}{\E_{\mu_h}\left[\langle w_h - \gT_h w_{h+1}, \phi(s, a)\rangle \mathbbm{1}_{s, a, h \in \gX_{\off}}\right]^2} \\
        &\quad \text{(due to Bellman completeness)}\\
        &= \max_h \sup_{\tilde w_h: \|\tilde w_h\|_2 \leq 2H\sqrt{d}} \sup_{\pi} \frac{\E_{d_h^{\pi}}\left[\langle \tilde w_{h}, \phi(s, a)\rangle \mathbbm{1}_{s, a, h \in \gX_{\off}}\right]^2}{\E_{\mu_h}\left[\langle \tilde w_{h}, \phi(s, a)\rangle \mathbbm{1}_{s, a, h \in \gX_{\off}}\right]^2} \\
        &= \max_h \sup_{\tilde w_h: \|\tilde w_h\|_2 \leq 2H\sqrt{d}} \sup_{\pi} \frac{\tilde w_{h}^\top \E_{d_h^{\pi}}\left[ \phi(s, a) \phi^\top(s, a) \mathbbm{1}_{s, a, h \in \gX_{\off}}\right] \tilde{w}_h}{\tilde w_{h}^\top \E_{\mu_h}\left[ \phi(s, a) \phi^\top(s, a) \mathbbm{1}_{s, a, h \in \gX_{\off}}\right] \tilde{w}_h}\\
        & \text{(due to the fact that $\gP_{\off}\phi(s, a) = \phi(s, a)$ for all $s, a, h \in \gX_{\off}$)}\\
        &\leq \max_h \sup_{\tilde w_h: \|\tilde w_h\|_2 \leq 2H\sqrt{d}} \sup_{\pi} \frac{\tilde w_{h}^\top \E_{d_h^{\pi}}\left[ \gP_{\off} \phi(s, a) (\gP_{\off}\phi(s, a))^\top \right] \tilde w_{h}}{\tilde w_{h}^\top \E_{\mu_h}\left[ \gP_{\off} \phi(s, a) (\gP_{\off}\phi(s, a))^\top \right] \tilde w_{h}} \\ 
        &\leq \max_h \frac{1}{\lambda_{d_{\off}}(\E_{\mu_h}\left[ \gP_{\off} \phi(s, a) (\gP_{\off}\phi(s, a))^\top \right])}.
    \end{align*}

We then bound the SEC through the distributional Bellman-Eluder dimension through Lemma \ref{lem:general_SEC}. It then suffices to bound the distributional Bellman-Eluder dimension as follows:
$$
    \max_h \dim_{\DE}((\gF_h - \gT_h \gF_{h+1})\mathbbm{1}_{(\cdot, h) \in \gX_{\on}}, \sD_h, \epsilon) \lesssim d_{\on}\log(H / \epsilon).
$$

The following lemma states, informally, that low Bellman rank families are MDPs such that the Bellman error can be written as the inner product of feature maps of the Bellman error and feature maps of the distributions. That is, the expected Bellman error can be written as such.
\begin{lem}
\label{lem:low_bellman_rank}
    There exist mappings $\psi: \gF_h - \gT_h \gF_{h+1} \mapsto \sR^{d_{\on}}$ and $\varphi: \sD_h \mapsto \sR^{d_{\on}}$ such that 
    $$
        \E_{(s, a) \sim d}[g(s, a) \mathbbm{1}_{(s, a, h) \in \gX_{\on}}] = \langle \psi(g), \varphi(d) \rangle \text{ for all } g\in\gF_h - \gT_h \gF_{h+1}, d \in \sD_h.
    $$
    Moreover, $\{\|\psi(f)\|_2,  \|\varphi(d)\|_2\} \leq \sqrt{2H\sqrt{d}}$.
\end{lem}

\begin{proof}
    For any $g \in \gF_{h}, \gT_h \gF_{h+1}$, we can write $g(s, a) = \langle w_{g}, \phi(s, a) \rangle$. Since $\phi(s, a) = V \phi_{\on}(s, a)$ for all $(s, a)$ such that there exists $(s, a, h) \in \gX_{\on}$, where $\phi_{\on} \in \sR^{d_{\on}}$ and $V \in \sR^{d \times d_{\on}}$ is a set of orthogonal basis of the subspace spanned by $\phi(\gX_{\on})$. Thus, we can write 
    $$
        g(s, a) \mathbbm{1}_{s, a, h \in \gX_{\on}} = \langle w_g', \phi_{\on}(s, a) \rangle \mathbbm{1}_{s, a, h \in \gX_{\on}}, \text{ where } w'_g \in \sR^{d_{\on}}.
    $$
    Therefore, we have 
    $$
        \E_{(s, a) \sim d}[g(s, a) \mathbbm{1}_{s, a, h \in \gX_{\on}}] = \langle w'_g, \E_{s, a \sim d} [ \mathbbm{1}_{s, a, h \in \gX_{\on}}  \phi_{\on}(s, a)]\rangle.
    $$
\end{proof}

We then use this lemma to bound the distributional Bellman error as follows.

The following proof is a minor modification from the proof of Proposition 11 of \cite{jin2020provably}.
Assume that $\dim_{\DE}((\gF_h - \gT_h \gF_{h+1})\mathbbm{1}_{(\cdot, h) \in \gX_{\on}}, \sD_h, \epsilon) \geq m$. Then let $d_1, \dots, d_m \in \sD_h$ be an $\epsilon$-independent sequence w.r.t. $(\gF_{h} - \gT_h \gF_{h+1})\mathbbm{1}_{(\cdot, h) \in \gX_{\on}}$. By Definition \ref{defn:DE_dim}, there exists $g_1, \dots g_m \in \gF_h - \gT_h \gF_{h+1}$ such that for all $i \in [m]$, $\sqrt{ \sum_{t = 1}^{i-1}(\E_{d_t}[g_i \mathbbm{1}_{(\cdot, h) \in \gX_{\on}}])^2} \leq \epsilon$ and $|\E_{d_i}[g_i \mathbbm{1}_{(\cdot, h) \in \gX_{\on}}]| \geq \epsilon$. By Lemma \ref{lem:low_bellman_rank}, this is equivalent to: for all $i \in [m]$,
$$
\sqrt{ \sum_{t = 1}^{i-1}(\langle \psi(g_i), \varphi(d_t) \rangle)^2} \leq \epsilon \quad \text{ and } \quad |\langle \psi(g_i), \varphi(d_i) \rangle| \geq \epsilon.
$$

For notational simplicity, define $\vx_i = \psi(g_i)$ and $\vz_i = \varphi(d_i)$ and $\mV_i = \sum_{t = 1}^{i-1} \vz_t \vz_t^{\top} + \frac{\epsilon^2}{{2H\sqrt{d}}} \mI$. The previous argument implies that for all $i \in [m]$,
$$
    \|\vx_i\|_{\mV_i} \leq \sqrt{2}\epsilon \text{ and } \|\vx_i\|_{\mV_i} \cdot \|\vz_i\|_{\mV_i^{-1}} \geq \epsilon.
$$
Therefore, we have $\|\vz_i\|_{\mV_i^{-1}} \geq 1/2$. By the matrix determinant lemma, 
\begin{align*}
    \det\left[\mV_m\right]
    &=\det\left[\mV_{m-1}\right]\left(1+\left\|\vz_m\right\|_{\mV_m^{-1}}^2\right) \\
    &\geq \frac{3}{2} \det\left[\mV_{m-1}\right] \\
    &\geq \ldots \\
    &\geq \det\left[\frac{\epsilon^2}{2H\sqrt{d}} \cdot \mI\right]\left(\frac{3}{2}\right)^{m-1}\\
    &=\left(\frac{\epsilon^2}{2H\sqrt{d}}\right)^{d_{\on}}\left(\frac{3}{2}\right)^{m-1}.
\end{align*}
    
On the other hand, 
$$
    \det[\mV_{m}] \leq \left(\operatorname{trace}(\mV_m)/d_{\on} \right)^{d_{\on}} \leq \left( \frac{2H\sqrt{d}(m-1)}{d_{\on}} + \frac{\epsilon^2}{2H\sqrt{d}} \right)^{d_{\on}}.
$$
Therefore, we obtain
$$
    \left(\frac{3}{2}\right)^{m-1} \leq\left(\frac{4H^2d(m-1)}{d_{\on} \epsilon^2}+1\right)^{d_{\on}}.
$$
Taking logarithm on both sides, we have
$$
    m \leq 4 \left[ 1+ d_{\on} \log\left(\frac{2H^2d(m-1)}{d_{\on}\epsilon^2}\right) + 1 \right],
$$
which implies that 
$$
    m \lesssim 1 + d_{\on}\log(H^2 / \epsilon^2 + 1).
$$

Combined with Lemma \ref{lem:general_SEC} and choosing $\epsilon = 1/\sqrt{T}$, we have
$$
    c_{\on}(\gF, \gX_{\on}, T) \lesssim d_{\on} \log(H T) \log(T).
$$

Finally, note that each $w_h \in \R^d$ is bounded in norm by $||w_h||_2\leq 2H\sqrt{d}$ by Lemma B.1 of \cite{jin2020provably}. We then find that $\gN_{\gF}(\rho) \leq \left(\frac{6H}{\rho}\right)^{dH}$, so $\log \gN_{\gF}(1/N) \leq dH\log(6NH\sqrt{d}).$

\end{proof}

\newpage
\subsection{Block MDP Case}
\label{app:block_case}

{\bf Proposition \ref{prop:BMDP_bound}.}{\it\quad 
    $c_{\off}(\gF, \gX_{\off}) \leq \sup_{\pi }\sup_{(u, a, h) \in \bar{\gX}_{\off}} {\bar{d}_h^{\pi}(u, a)}/{\bar{\mu}_h^{\pi}(u, a)}$ and $c_{\on}(\gF, \gX_{\on}, T) = \gO(\max_{h}|\bar{\gX}_{\on, h}| \log(N_{\on}))$ in a BMDP with Bellman-complete $\gF$. With probability $1-\delta$,
    $$
    \Reg(N_{\on})
    = \tilde{O}\left(\inf_{\gX_{\on}, \gX_{\off}} \left(\sqrt{\beta H^4\frac{N_{\on}^2}{N_{\off}} \cdot \sup_{\pi }\sup_{(u, a, h) \in \bar{\gX}_{\off}} \frac{\bar{d}_h^{\pi}(u, a)}{\bar{\mu}_h^{\pi}(u, a)}}  + \sqrt{\max_{h}|\bar{\gX}_{\on, h}|\beta H^4 N_{\on} \log(N_{\on})}\right)\right),
    $$
    where 
    $\beta = c_1\left(\log \left[N H \mathcal{N}_{\mathcal{F}}(1/N) / \delta\right]\right)$ for some constant $c_1$ with $N = N_{\on} + N_{\off}$.
}
\begin{proof}
    The offline partition can be upper bounded by
    \begin{align*}
        c_{\off}(\gF, \gX_{\off}) 
        & = \max_h \sup_{\pi}  \sup_{f\in\gF_{\off}} \frac{\E_{d_h^{\pi}}[(g)^2(s, a)]}{\E_{\mu_h}[(g)^2(s, a)]} \\
        & = \max_h \sup_{\pi} \sup_{f\in\gF_{\off}} \frac{\E_{(u, a) \sim \bar{d}_h^{\pi}}[(g)^2(f^*(u), a)]}{\E_{(u, a) \sim \bar{\mu}_h}[(g)^2(f^*(u), a)]}\\
        & = \max_h \sup_{\pi} \sup_{f\in\gF} \frac{\E_{(u, a) \sim \bar{d}_h^{\pi}}[(g)^2(f^*(u), a) \mathbbm{1}_{u, a, h \in \bar{\gX}_{\off}}]}{\E_{(u, a) \sim \bar{\mu}_h}[(g)^2(f^*(u), a)\mathbbm{1}_{u, a, h \in \bar{\gX}_{\off}}]} \\
        & \leq \sup_{\pi} \sup_{u, a, h \in \bar{\gX}_{\off}}\frac{\bar{d}_h^{\pi}(u, a)}{\bar{\mu}_h(u, a)}.
    \end{align*}
    For the online partition, we have 
    \begin{align*}
        c_{\on}(\gF, \gX_{\on}, T) &\lesssim \log(T) \max_{h} \inf_{\mu_h \in \Delta(\gS \times \gA)} \sup_{\pi} \frac{\E_{(s, a) \sim d_h^{\pi}}[\mathbbm{1}_{(s, a, h) \in \gX_{\on}}]}{\E_{(s, a) \sim \mu_h^{\pi}}[\mathbbm{1}_{(s, a, h) \in \gX_{\on}}]}\\
        &\leq \log(T) \max_{h} \inf_{\bar{\mu}_h \in \Delta(\gU \times \gA)} \sup_{\pi} \frac{\E_{(u, a) \sim \bar{d}_h^{\pi}}[\mathbbm{1}_{(u, a, h) \in \gX_{\on}}]}{\E_{(u, a) \sim \bar{\mu}_h^{\pi}}[\mathbbm{1}_{(u, a, h) \in \gX_{\on}}]}\\
        &\leq \log(T) \max_h |\bar{\gX}_{\on, h}|
    \end{align*}
\end{proof}

\newpage
\section{General Recipe}
\label{app:general_case}

{\bf Proposition \ref{prop:general_recipe}.}{\quad \it
    Let $\gL$ be a general online learning algorithm that satisfies the following conditions:
    \begin{enumerate}
        \item $\gL$ admits the regret decomposition $\Reg_{\gL}(T) \leq \sum_{t = 1}^T \sum_{h = 1}^H \E_{(s, a) \sim d^{(t)}_h}[\delta_h^{t}(s, a)]$ for some collection of random functions\footnote{This is often the Bellman error in the case of MDPs.} $(\delta_h^{t})_{h = 1}^H$ with each $\delta_h^t$ a mapping from $\gX \mapsto \sR$; 
        \item it holds with probability at least $1-\delta$ that
        $$
            \sum_{t = 1}^T \sum_{h = 1}^H \left(N_{\off} \E_{(s, a) \sim \mu_h}[\delta_h^t(s, a)^2] + \sum_{i = 1}^{t-1} \E_{(s, a) \sim d^{i}_h}[ \delta_h^{t}(s, a)^2]\right) \leq \beta(\delta, H);
        $$
        \item there exists a function $c_{\on}: \gP(\gX) \times \sN$ such that for any $\gX' \subset \gX$, it holds with a probability at least $1-\delta$
        $$
            \sum_{t = 1}^T \sum_{h = 1}^H \E_{(s, a) \sim d^{(t)}_h}[\delta_h^t(s, a) \mathbbm{1}{(x, a, h) \in \gX'}] = \gO(c_{\on}(\gX', T) \beta(\delta, H) H^\gamma T)^{\xi},
        $$ for some $\xi \in (0, 1)$, $\gamma \in \mathbb{Z}_{\geq 0}$, and where $\beta:(0, 1) \mapsto \sR$ is some measure of complexity of the algorithm and its dependence on the probability of failure $\delta$;
        
        \item for any $\gX' \subset \gX$, there exists a measure of coverage $$
            c_{\off}(\gX') \coloneqq \sup_{h \in [H]} \sup_{\pi} \frac{\E_{d_h^{\pi}}[\delta_h^t(s, a) \mathbbm{1}(s,a,h \in \gX')]}{\E_{\mu_h}[\delta_h^t(s, a)\mathbbm{1}(s,a,h \in \gX')]} \footnote{We set $0/0$ as 0.}.
        $$
    \end{enumerate}
    Then, the algorithm $\gL$ satisfies the following regret bound:
    $$
        \Reg_{\gL}(T) = \gO\left(\inf_{\gX_{\on}, \gX_{\off}}(c_{\on}(\gX_{\on}, T) \beta(\delta, H) H^\gamma T)^{\xi}  + H \sqrt{\beta(\delta, H) \cdot c_{\off}(\gX_{\off})\cdot \frac{N_{\on}^2}{N_{\off}}}\right).
    $$
}

\begin{proof}
    We first use the regret decomposition in Condition 1 to obtain
    \begin{align*}
        \Reg_{\gL} 
        &\leq \sum_{t = 1}^T \sum_{h = 1}^H \E_{(s, a) \sim d^{(t)}_h}[\delta_h^{t}(s, a)]\\
        &= \sum_{t = 1}^T \sum_{h = 1}^H \E_{(s, a) \sim d^{(t)}_h}[\delta_h^{t}(s, a) \mathbbm{1}(s,a,h \in \gX_{\on})] + \sum_{t = 1}^T \sum_{h = 1}^H \E_{(s, a) \sim d^{(t)}_h}[\delta_h^{t}(s, a) \mathbbm{1}(s,a,h \in \gX_{\off})].
    \end{align*}
    The regret bound on the online partition follows from condition 2, as we have $\sum_{t = 1}^T \sum_{h = 1}^H \E_{(s, a) \sim d^{(t)}_h}[\delta_h^t(s, a) \mathbbm{1}{(x, a, h) \in \gX_{\on}}] = \gO(c_{\on}(\gX_{\on}, T) \beta(\delta, H) H^\gamma T)^{\xi}$.
    
    We then upper bound the regret of the offline term. We denote $\delta_h^{t}(s, a) \mathbbm{1}(s,a,h \in \gX_{\off})$ by $\tilde{\delta}_h^{t}(s, a)$. To proceed, we have
    \begin{align*}
        &   \sum_{t = 1}^T \sum_{h = 1}^H \E_{(s, a) \sim d^{(t)}_h}[\tilde{\delta}_h^{t}(s, a)]\\
        =&  \sum_{t = 1}^T \sum_{h = 1}^H \E_{(s, a) \sim d^{(t)}_h}\left[\tilde{\delta}_h^{t}(s, a) \left(\frac{N_{\off}\E_{\mu_{h}}[\tilde{\delta}_h^{t}(s, a)^2] +  \sum_{i=1}^{t-1}\E_{d_h^{(i)}}[\tilde{\delta}_h^{t}(s, a)^2]}{N_{\off}\E_{\mu_{h}}[\tilde{\delta}_h^{t}(s, a)^2] + \sum_{i=1}^{t-1}\E_{d_h^{(i)}}[\tilde{\delta}_h^{t}(s, a)^2]}\right)^{1 / 2}\right] \\
        \leq& \sqrt{\sum_{t = 1}^T \sum_{h = 1}^H \E_{(s, a) \sim d^{(t)}_h} \frac{\tilde{\delta}_h^{t}(s, a)^2}{N_{\off}\E_{\mu_{h}}[\tilde{\delta}_h^{t}(s, a)^2] +  \sum_{i=1}^{t-1}\E_{d_h^{(i)}}[\tilde{\delta}_h^{t}(s, a)^2]}} \\
        &\hspace{5mm}\sqrt{\sum_{t = 1}^T \sum_{h = 1}^H \left(N_{\off}\E_{\mu_{h}}[\tilde{\delta}_h^{t}(s, a)^2] +  \sum_{i=1}^{t-1}\E_{d_h^{(i)}}[\tilde{\delta}_h^{t}(s, a)^2]\right)} \\
        \leq& \sqrt{ \frac{\sum_{t = 1}^T \sum_{h = 1}^H \E_{(s, a) \sim d^{(t)}_h}\tilde{\delta}_h^{t}(s, a)^2}{N_{\off}\E_{\mu_{h}}[\tilde{\delta}_h^{t}(s, a)^2]}}  \sqrt{\sum_{t = 1}^T \sum_{h = 1}^H \left(N_{\off}\E_{\mu_{h}}[\tilde{\delta}_h^{t}(s, a)^2] +  \sum_{i=1}^{t-1}\E_{d_h^{(i)}}[\tilde{\delta}_h^{t}(s, a)^2]\right)} \\
        \leq& \sqrt{TH c_{\off}(\gT\gF, \bar d)/N_{\off}} \sqrt{\sum_{t = 1}^T \sum_{h = 1}^H \left(N_{\off}\E_{\mu_{h}}[\tilde{\delta}_h^{t}(s, a)^2] +  \sum_{i=1}^{t-1}\E_{d_h^{(i)}}[\tilde{\delta}_h^{t}(s, a)^2]\right)} \\
        \leq& H\sqrt{N_{\on} c_{\off}(\gX_{\off})N_{\on}/N_{\off}\beta(\gamma, \gX_{\off}) }
    \end{align*}
\end{proof}

\newpage
\section{Technical and Miscellaneous Lemmas}

\subsection{Lemma \ref{lem:ellipsoid}}
\begin{lem}[Lemma D.2 in \cite{jin2020provably}]
\label{lem:ellipsoid}
    Let $\left\{\phi_t\right\}_{t \geq 0}$ be a bounded sequence in $\mathbb{R}^d$ satisfying $\sup _{t>0}\left\|\phi_t\right\| \leq 1$. Let $\Lambda_0 \in \mathbb{R}^{d \times d}$ be a positive definite matrix. For any $t \geq 0$, we define $\Lambda_t=\Lambda_0+\sum_{j=1}^t {\phi}_j \phi_j^{\top}$. Then, if the smallest eigenvalue of $\Lambda_0$ satisfies $\lambda_{\text{min}}(\Lambda_0) \geq 1$, we have
    $$
        \log \left[\frac{\det\left(\Lambda_t\right)}{\det\left(\Lambda_0\right)}\right] \leq \sum_{j=1}^t \phi_j^{\top} \Lambda_{j-1}^{-1} \boldsymbol{\phi}_j \leq 2 \log \left[\frac{\det\left(\Lambda_t\right)}{\det\left(\Lambda_0\right)}\right].
    $$
\end{lem}

\subsection{Proof of Lemma \ref{lem:good_event}}

\label{lem:good_event_proof}
{\bf Lemma \ref{lem:good_event}. }{
\it
With a probability at least $1-\delta$, for all $t \in [N_{\on}]$, we have that for all $h=1,...,H$
\begin{equation*}
    \text{(i) } Q^*_h \in \mathcal{F}^{(t)}_h, \text{ (ii) } \E_{\tilde{d}_h^{(t)}}[(\delta_h^{(t)})^2] \leq O(H^2\beta),
\end{equation*}
by choosing $\beta = c_1\left(\log \left[N H \mathcal{N}_{\mathcal{F}}(\rho) / \delta\right]+ N \rho\right)$ for some constant $c_1$. 
}
\begin{proof}
    Lemma 44 in \cite{jin2021bellmaneluder} showed that with high probability: (i) any function $f^{(\tau)}$ in the confidence set has low Bellman-error over the collected Datasets $\gD_1^{(\tau)}, \dots, \gD_{H}^{(\tau)}$ as well as the distributions from which $\gD_{1}^{(\tau)}, \dots, \gD_{H}^{(\tau)}$ are sampled; (ii) the optimal value function is inside the confidence set. We use this to our setting as follows, with the intuition being that we pre-append a sequence of functions generated from the offline dataset from samples $1,...,N_{\off}$ to the $N_{\on}$ sequence. 
    
    That is, consider $N = N_{\off} + N_{\on}$, and consider a set of functions $\bar f^{\tau}, \tau = 1, \dots, N$, which we define as follows. For each $\tau \in [N_{\off}]$, define $\bar f^{\tau}$ to be any arbitrary function in the confidence set of functions constructed by the first $\tau$ episodes of the offline dataset (we can set an arbitrary order for the episodes in the offline dataset). For each $\tau = N_{\off}+1, \dots, N$, define $\bar f^{\tau} := f^{(\tau - N_{\off})} \in \gF^{(t)}$. As Lemma 44 in \cite{jin2021bellmaneluder} shows that (i) and (ii) hold for all $\tau \in [N]$, they must also hold for all $\tau = N_{\off}+1, \dots, N$.
\end{proof}

\subsection{Lemmas on Coverage}

\begin{lem}[Restricted Per-State-Action Elliptic Potential Lemma (modified from \cite{xie2022role})]
    \label{lem:restricted-elliptic}
    Consider an arbitrary sequence of densities $d_h^{(1)},...,d_h^{(T)}$, and a partition $\gX_{\off}, \gX_{\on}$ of $\gX$. Define 
    $$\mu^*_h := \argmin_{\mu_h \in \Delta(\gS \times \gA)} \sup_{\pi} \frac{\E_{(s, a) \sim d_h^{\pi}}[\mathbbm{1}_{(s, a, h) \in \gX_{\on}}]}{\E_{(s, a) \sim \mu_h^{\pi}}[\mathbbm{1}_{(s, a, h) \in \gX_{\on}}]}, \;\; c_h(\gX_{\on}) := \inf_{\mu_h \in \Delta(\gS \times \gA)} \sup_{\pi} \frac{\E_{(s, a) \sim d_h^{\pi}}[\mathbbm{1}_{(s, a, h) \in \gX_{\on}}]}{\E_{(s, a) \sim \mu_h^{\pi}}[\mathbbm{1}_{(s, a, h) \in \gX_{\on}}]}.$$

    Observe that $d_h^{(t)}(s,a) / \mu^*_h(s,a) \leq c_h(\gX_{\on})$ for all $(s,a,h) \in \gX_{\on}$. For all $(s,a,h)\in \gX_{\on}$, we have that 
    $$\sum_{t = 1}^T\frac{d_h^{(t)}(s,a)\mathbbm{1}((s,a) \in \gX_{\on})}{\left(c_h(\gX_{\on})\cdot\mu_h^*(s,a) + \sum_{i=1}^{t-1} d_h^{(i)}(s,a)\right)\mathbbm{1}((s,a) \in \gX_{\on})} \leq \gO(\log(T))$$
\end{lem}
\begin{proof}
    The lemma, and the proof, is slightly modified from Lemma 4 of \cite{xie2022role} to account for our restriction to the online partition, as well as the fact that the restricted distributions are no longer distributions. Observe that $d_h^{(t)} \leq c_h(\gX_{\on})\cdot\mu_h^*$ by definition, so the quantity inside the sum is within $[0,1]$. Using the fact for any $u \in[0,1], u \leq 2 \log (1+u)$, we have
    \begin{align*}
        &\sum_{t = 1}^T\frac{d_h^{(t)}(s,a)\mathbbm{1}((s,a) \in \gX_{\on})}{\left(c_h(\gX_{\on})\cdot\mu_h^*(s,a) + \sum_{i=1}^{t-1} d_h^{(i)}(s,a)\right)\mathbbm{1}((s,a) \in \gX_{\on})} \\
        &\leq 2\sum_{t = 1}^T\log\left(1+\frac{d_h^{(t)}(s,a)\mathbbm{1}((s,a) \in \gX_{\on})}{\left(c_h(\gX_{\on})\cdot\mu_h^*(s,a) + \sum_{i=1}^{t-1} d_h^{(i)}(s,a)\right)\mathbbm{1}((s,a) \in \gX_{\on})}\right) \\
        &\leq 2\sum_{t = 1}^T\log\left(\frac{\left(c_h(\gX_{\on})\cdot\mu_h^*(s,a) + \sum_{i=1}^{t} d_h^{(i)}(s,a)\right)\mathbbm{1}((s,a) \in \gX_{\on})}{\left(c_h(\gX_{\on})\cdot\mu_h^*(s,a) + \sum_{i=1}^{t-1} d_h^{(i)}(s,a)\right)\mathbbm{1}((s,a) \in \gX_{\on})}\right) \\
        &= 2\log\left(\prod_{t=1}^T\frac{\left(c_h(\gX_{\on})\cdot\mu_h^*(s,a) + \sum_{i=1}^{t} d_h^{(i)}(s,a)\right)\mathbbm{1}((s,a) \in \gX_{\on})}{\left(c_h(\gX_{\on})\cdot\mu_h^*(s,a) + \sum_{i=1}^{t-1} d_h^{(i)}(s,a)\right)\mathbbm{1}((s,a) \in \gX_{\on})}\right) \\
        &= 2\log\left(\frac{\left(c_h(\gX_{\on})\cdot\mu_h^*(s,a) + \sum_{t=1}^T d_h^{(i)}(s,a)\right)\mathbbm{1}((s,a) \in \gX_{\on})}{\left(c_h(\gX_{\on})\cdot\mu_h^*(s,a)\right)\mathbbm{1}((s,a) \in \gX_{\on})}\right) \\
        &\leq 2\log\left(\frac{(1+T)\left(c_h(\gX_{\on})\cdot\mu_h^*(s,a)\right)\mathbbm{1}((s,a) \in \gX_{\on})}{\left(c_h(\gX_{\on})\cdot\mu_h^*(s,a)\right)\mathbbm{1}((s,a) \in \gX_{\on})}\right) \\
        &\leq 2\log(1+T),
    \end{align*}
    where the last line follows from the observation that $d_h^{(t)}(s,a) / \mu^*_h(s,a) \leq c_h(\gX_{\on})$ for all $(s,a,h) \in \gX_{\on}$.
\end{proof}

\newpage
\section{Additional Figures}
\label{app:experiments}

\begin{figure}[H]
    \centering
    \includegraphics[scale=0.5]{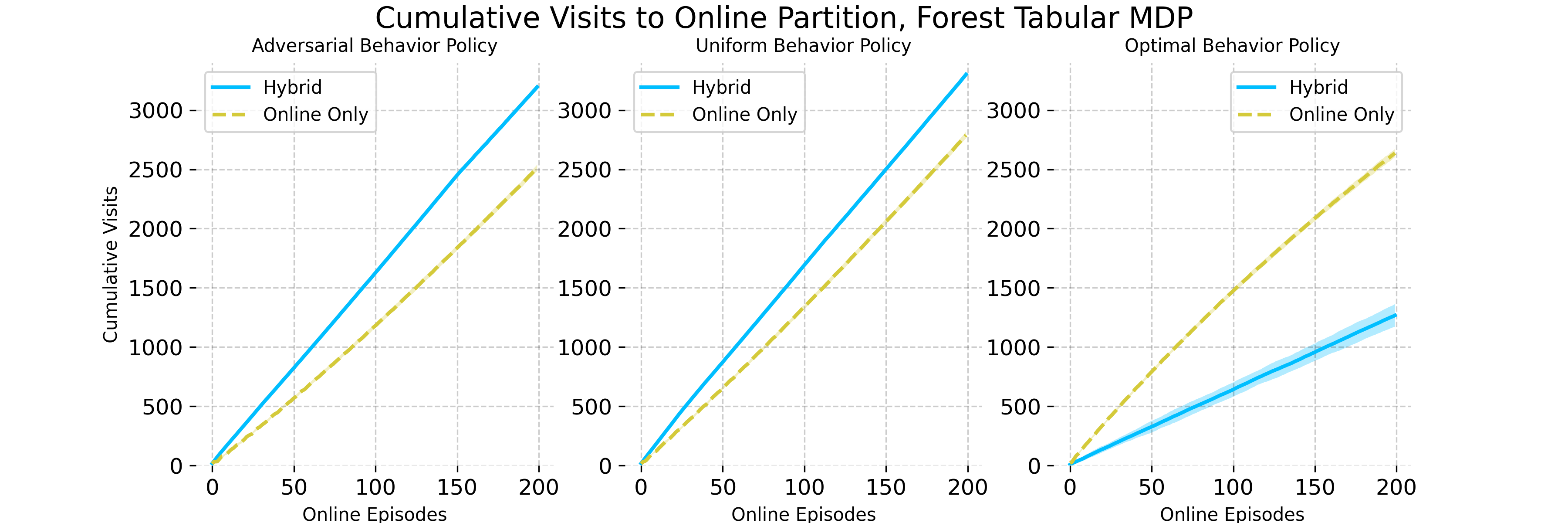}
    \vspace{0.1in}
    \includegraphics[scale=0.5]{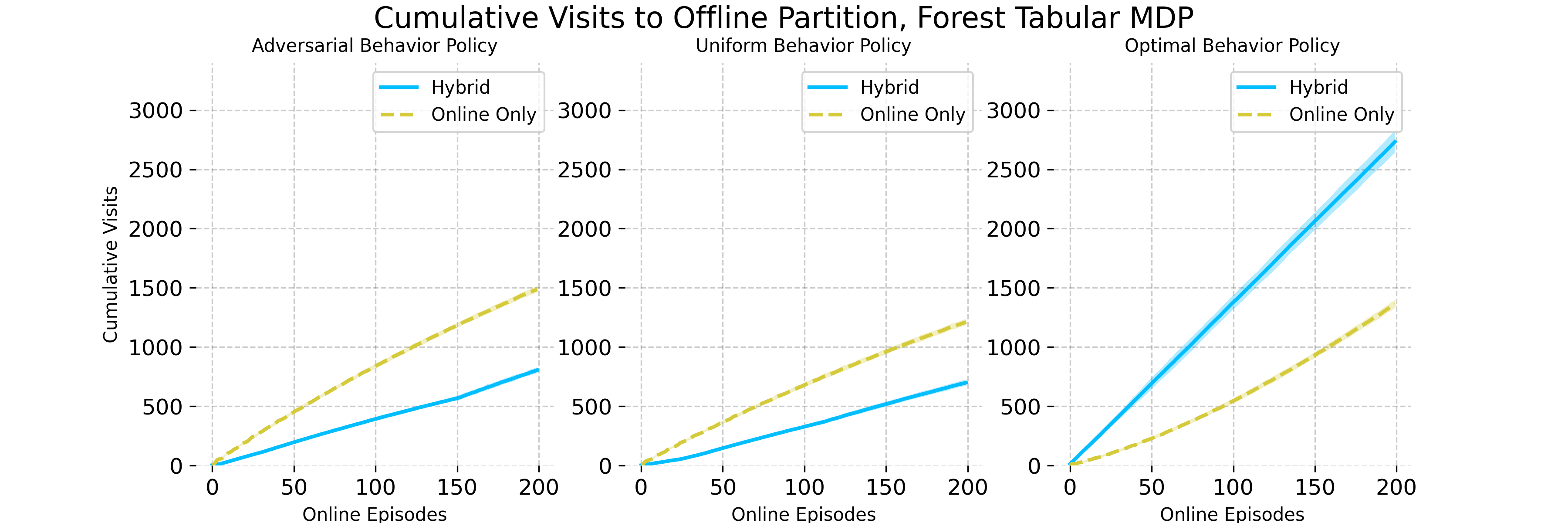}
    \caption{Cumulative visits to the offline and online partitions over the $200$ online episodes of horizon $20$. When the behavior policy is poor or middling, the hybrid algorithm visits the online partition more and the offline partition less than the online-only algorithm does. When the behavior policy is optimal, the converse occurs, as the model parameters in UCBVI \citep{azar2017minimax} are warm-started by estimating them from the offline dataset, enabling the hybrid algorithm to learn that the offline partition contains the good state-action pairs. Solid lines indicate the mean over $30$ trials, and the shaded area denotes a confidence interval of $1.96$ sample standard deviations.}
    \label{fig:tabular-visits}
\end{figure}

\newpage 

\begin{figure}[H]
    \centering
    \includegraphics[scale=0.5]{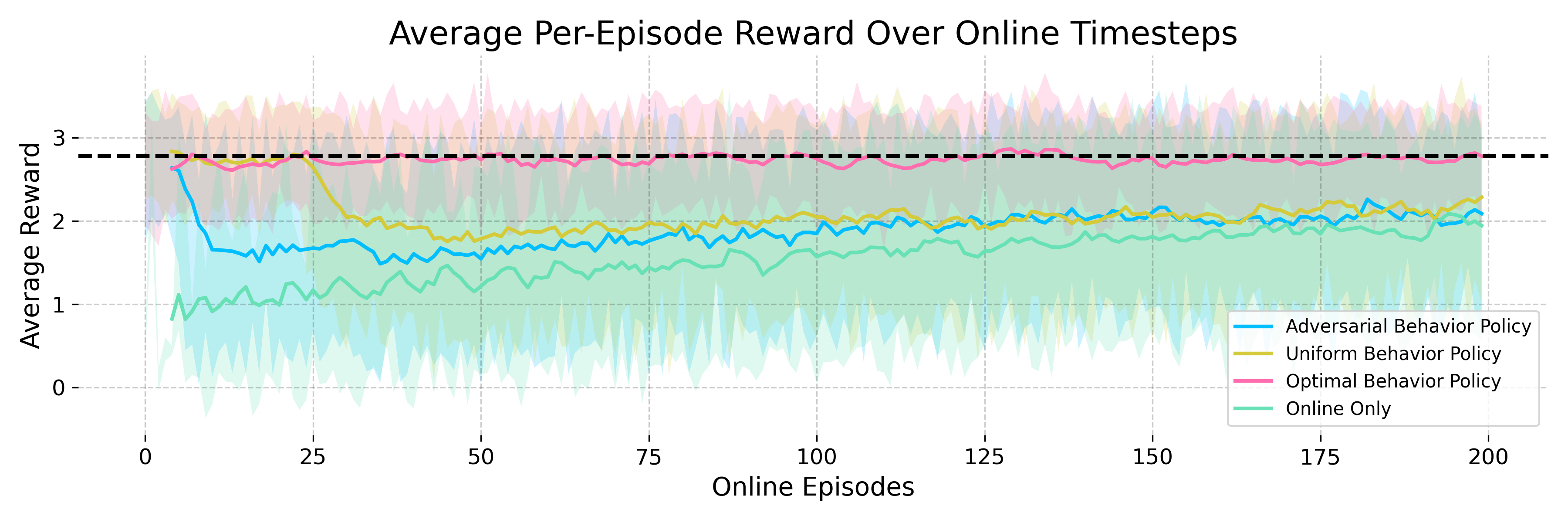}
    \caption{Average reward over $200$ episodes from running UCBVI \citep{azar2017minimax} both in its original form and initialized with an offline dataset. When the behavior policy is optimal, the hybrid algorithm learns the optimal policy quickly. When it is not, we still gain an advantage over online-only learning, even when the behavior policy is adversarial, even though in these cases $200$ episodes are not sufficient to learn the optimal policy. Incidentally, the hybrid algorithm with poor behavior policies has a high reward at the start, but faces a drop in performance as it explores other states and actions due to the very large exploration bonus we chose to encourage exploration. Results averaged over $30$ trials, with $1$ standard deviation-wide shaded areas. }
    \label{fig:reward-tabular}
\end{figure}

\begin{figure}[H]
    \centering
    \includegraphics[scale=0.5]{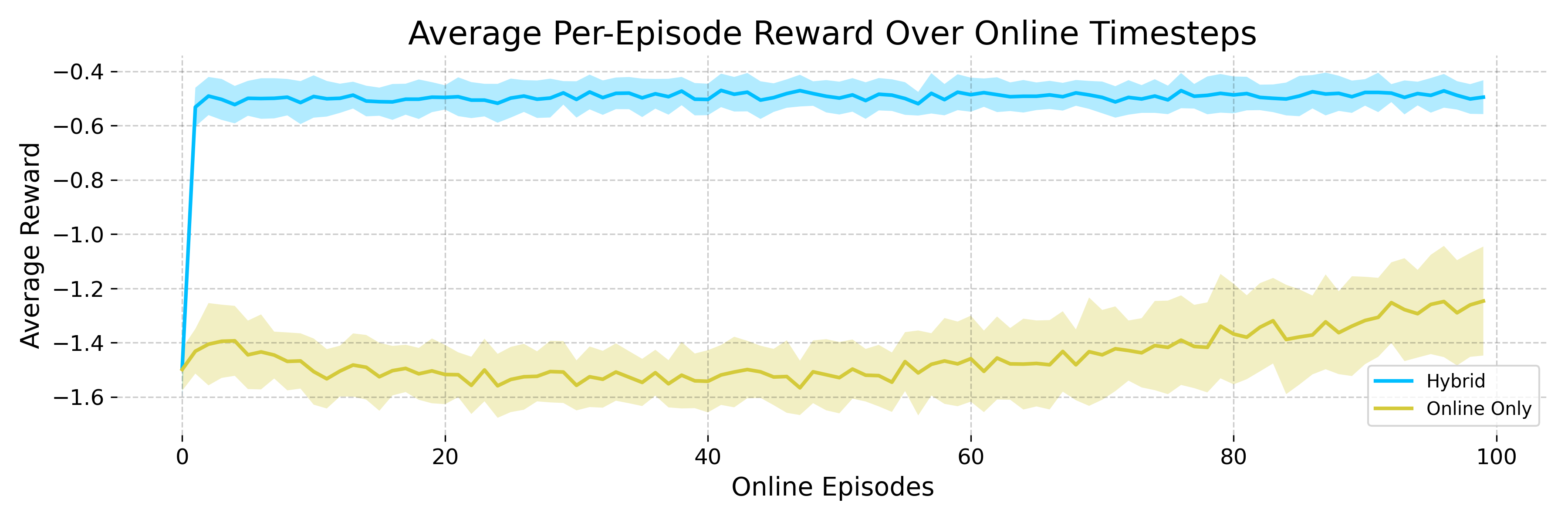}
    \caption{Average reward of each episode when running LSVI-UCB \citep{jin2020provably} in its original form and initialized with an offline dataset. Results averaged over $30$ trials, with $1$ standard deviation-wide shaded areas. The hybrid version approaches the optimal weights almost instantaneously, while the online-only version takes many more episodes to do the same.}
    \label{fig:reward-linear}
\end{figure}

\end{document}